\documentclass[11pt,oneside,letterpaper]{article}

\usepackage{natbib}
\usepackage[margin=1in]{geometry}

\usepackage{hyperref}       
\usepackage{url}            
\usepackage{booktabs}       
\usepackage{amsfonts}       
\usepackage{nicefrac}       
\usepackage{microtype}      
\usepackage{xcolor}         
\usepackage{cuted}

\usepackage{wrapfig}
\usepackage{caption}
\usepackage{comment}
\usepackage{amsmath}
\usepackage{amssymb}
\usepackage{mathtools}
\usepackage{amsthm}
\usepackage{bm}
\usepackage{algorithm}
\usepackage{algorithmic}

\usepackage[capitalize,noabbrev]{cleveref}

\usepackage{subfigure}

\theoremstyle{plain}

\newtheorem{proposition}{Proposition}
\newtheorem{lemma}{Lemma}
\newtheorem*{lemma*}{Lemma}

\theoremstyle{definition}
\newtheorem{definition}{Definition}
\newtheorem{assumption}{Assumption}

\newtheorem{example}{Example}

\usepackage{dsfont}
\usepackage{palatino}

\usepackage{cancel}

\usepackage{graphicx}
\usepackage{multirow}
\usepackage{float}
\usepackage{pgffor}
\foreach \x in {a,...,z}{%
\expandafter\xdef\csname vec\x \endcsname{\noexpand\ensuremath{\noexpand\bm{\x}}}
}
\foreach \x in {A,...,Z}{%
\expandafter\xdef\csname vec\x \endcsname{\noexpand\ensuremath{\noexpand\bm{\x}}}
}
\foreach \x in {A,...,Z}{%
\expandafter\xdef\csname c\x \endcsname{\noexpand\ensuremath{\noexpand\mathcal{\x}}}
}
\foreach \x in {A,...,Z}{%
\expandafter\xdef\csname bb\x \endcsname{\noexpand\ensuremath{\noexpand\mathbb{\x}}}
}
\usepackage{xspace}

\newcommand{\btheta}{\bm{\theta}}

\usepackage{pifont}

\renewcommand{\Pr}{P}

\usepackage{tikz}
\usetikzlibrary{arrows.meta}
\usetikzlibrary{positioning}
\usetikzlibrary{decorations}

\usepackage{amsmath}
\usepackage{xspace}
\newcommand{\bdp}{BCDP\xspace}

\usepackage{xcolor}  
\usepackage{hyperref}
\usepackage[normalem]{ulem}

\definecolor{Mahogany}{RGB}{192, 64, 0}   
\definecolor{OliveGreen}{RGB}{107, 142, 35} 
\definecolor{Violet}{RGB}{148, 0, 211}  
\definecolor{Teal}{RGB}{0, 128, 128} 
\definecolor{DarkerViolet}{RGB}{128, 0, 180}

\definecolor{darkteal}{rgb}{0.0, 0.4, 0.4}
\hypersetup{
    colorlinks,
    linkcolor= {DarkerViolet},
    citecolor={darkteal},
    urlcolor={OliveGreen}
}

\newcommand{\notes}{1}
\ifnum\notes=1
    \newcommand{\mynote}[3]{\marginpar{{\tiny \parbox{0.8in}{\color{#2} \sf {#1}: {#3}}}}}
    \newcommand{\inlinenote}[3]{{\bfseries \color{#2} {#1}: #3}}
    
    \newcommand{\mystrikeout}[2]{{\color{#1} \sout{#2}}} 
\else
    \newcommand{\mynote}[3]{}
    \newcommand{\inlinenote}[3]{}
    
    \newcommand{\mystrikeout}[2]{}
\fi

\usepackage{thm-restate}

\renewcommand{\epsilon}{\varepsilon}

\newcommand{\red}[1]{\textcolor{red}{ #1}}
\newcommand{\violet}[1]{\textcolor{violet}{ #1}}

\usepackage{comment}
\usepackage{setspace}
\begin{document}

\title{Enhancing Feature-Specific Data Protection via \\ Bayesian Coordinate Differential Privacy~\thanks{The authors are listed in alphabetical order.}}

\author{Maryam Aliakbarpour \thanks{Department of Computer Science \& Ken Kennedy Institute, Rice University 
} 
\and 
Syomantak Chaudhuri  \thanks{Department of Electrical Engineering and Computer Sciences, University of California, Berkeley} 
\and 
Thomas A. Courtade $^{\text{\textdaggerdbl}}$
\and 
Alireza Fallah $^{\text{\textdaggerdbl}}$
\and 
Michael I. Jordan \thanks{
Department of Electrical Engineering and Computer Sciences \& 
Department of Statistics,
University of California, Berkeley;
INRIA, Paris}}

\date{October 23, 2024}
\maketitle

\sloppy

\begin{abstract}
Local Differential Privacy (LDP) offers strong privacy guarantees without requiring users to trust external parties. However, LDP applies uniform protection to all data features, including less sensitive ones, which degrades performance of downstream tasks. To overcome this limitation, we propose a Bayesian framework, Bayesian Coordinate Differential Privacy (BCDP), that enables feature-specific privacy quantification. This more nuanced approach complements LDP by adjusting privacy protection according to the sensitivity of each feature, enabling improved performance of downstream tasks without compromising privacy. We characterize the properties of BCDP and articulate its connections with  standard non-Bayesian privacy frameworks. We further apply our BCDP framework to the problems of private mean estimation and ordinary least-squares regression. The BCDP-based approach obtains improved accuracy compared to a purely LDP-based approach, without compromising on privacy.    
\end{abstract}

\section{Introduction}

The rapid expansion of machine learning applications across various sectors has fueled an unprecedented demand for data collection. With the increase in data accumulation, user privacy concerns also intensify. Differential privacy (DP) has emerged as a prominent framework that enables algorithms to utilize data while preserving the privacy of individuals who provide it \cite{warner1965randomized, evfimievski2003limiting, dwork2006calibrating,  dwork2014algorithmic, vadhan2017complexity, Desfontaines2020RealWorldDP}. A variant of DP, local differential privacy (LDP), suitable for distributed settings where users do not trust any central authority has been extensively studied~\citep{warner1965randomized, evfimievski2003limiting, beimel2008distributed, kasiviswanathan2011can, ChanSS12}. In this framework, each user shares an obfuscated version of their datapoint, meaningful in aggregate to the central server yet concealing individual details. Leading companies, including Google~\citep{erlingsson2014rappor}, Microsoft~\citep{ding2017collecting}, and Apple~\citep{apple2017learning, thakurta2017learning}, have employed this approach for data collection.

In the LDP framework, a privacy parameter must be set---typically by the data collector---that determines the level of protection for user information. This decision is particularly challenging in high-dimensional settings, where user data includes multiple attributes such as name, date of birth, and social security number, each with differing levels of sensitivity. One common strategy is to allow the most sensitive data feature to dictate the level of privacy, providing the highest level of protection. Alternatively, separate private mechanisms can be applied to each feature according to its required privacy level. However, this approach fails when there is correlation between features. For instance, setting a lower level of privacy for a feature that is consistently aligned with a more sensitive one risks compromising the privacy of the sensitive feature.

This complexity raises a compelling question: Can we tailor privacy protection to specific data features if the correlation among data coordinates is controlled, without defaulting to the privacy level of the most sensitive coordinate? In other words, can we leverage users' lower sensitivity preferences for certain features to reduce the error in our inference algorithms, provided the correlation with sensitive coordinates is weak?

\paragraph{Our contributions:} 
Our main contribution is addressing the challenge of feature-specific data protection from a Bayesian perspective. We present a novel Bayesian privacy framework, which we refer to as {\em Bayesian Coordinate Differential Privacy} (BCDP). BCDP  complements LDP and allows a tailored privacy demand per feature. In essence, BCDP ensures that the odds of inferring a specific feature (also referred to as a  coordinate) remain nearly unchanged before and after observing the outcome of the private mechanism. The framework’s parameters control the extent to which the odds remain unchanged, enabling varying levels of privacy per feature.
As part of developing this new framework, we explore its formal relationship with other related DP notions and examine standard DP properties, such as {\em post-processing} and {\em composition}.

To demonstrate the efficacy of our framework, we study two fundamental problems in machine learning under a baseline level of LDP privacy, coupled with more stringent per-coordinate BCDP constraints. First, we address {\em multivariate mean estimation}. Our proposed algorithm satisfies BCDP constraints while outperforming a standard LDP algorithm, where the privacy level is determined by the most sensitive coordinate. Our experimental results confirm this advantage. Second, we present an algorithm for {\em ordinary least squares regression} under both LDP and BCDP constraints.

A key technical component of our algorithm for these two problems is a mechanism that builds on {\em any} off-the-shelf LDP mechanism and can be applied to various problems. Our algorithm involves making parallel queries to the LDP mechanism with a series of carefully selected privacy parameters where the $i$-th query omits the $i-1$ most sensitive coordinates. The outputs are then aggregated to produce the most accurate estimator for each coordinate. 

\paragraph{Bayesian Feature-Specific  Protection:}
We now give a high-level discussion of privacy in the Bayesian framework to motivate  our approach. Formal definitions are  in Section~\ref{sec:formulation}.
Consider a  set of users, each possessing a high-dimensional data point $\vecx$ with varying privacy protection needs across different features. For instance, $\vecx$ may contain genetic markers for diseases, and the $i$-th coordinate $x_i$ could reveal Alzheimer’s markers that the user wishes to keep strictly confidential. In a local privacy setting, users apply a locally private mechanism $M$ to obtain obfuscated data $M(\vecx)$, for public   aggregation and analysis. Our focus is on understanding what can be inferred about $\vecx$, especially the sensitive feature $x_i$, from the obfuscated data.

Now, imagine an adversary placing a bet on a probabilistic event related to $\vecx$. Without extra information, their chances of guessing correctly are based on prior knowledge of the population. However, access to $M(\vecx)$ could prove advantageous for inference. Privacy in the Bayesian model is measured by how much $M(\vecx)$ improves the adversary’s chances. In other words, the privacy parameter limits how much their odds can shift.

The standard Bayesian interpretation of LDP ensures uniform privacy across all data features. In contrast, our framework BCDP enhances LDP by enabling feature-specific privacy controls. That is, BCDP restricts adversarial inference of events concerning  individual coordinates. For example, BCDP ensures that the advantage an adversary gains from observing $M(\vecx)$ when betting on whether a person has the Alzheimer’s trait (i.e., $x_i = 1$) is smaller than what would be achieved under a less stringent LDP constraint.

We remark that BCDP does not replace LDP, since it lacks global data protection. Also, correlation plays a crucial role in BCDP's effectiveness: if a sensitive feature is highly correlated with less sensitive features, increasing protection for the sensitive feature imposes stricter protection for other features as well. As the correlation between features increases, BCDP's guarantees converge to the  uniform guarantees of LDP.

\paragraph{Related Work:}
Statistical inference with LDP guarantees has been extensively studied. For comprehensive surveys  see \cite{xiong2020comprehensive, yang2023local}.
Below, we discuss the results that are most relevant to our setup.

Several papers consider the uneven protection of privacy across features. For instance, \citet{ghazi2022algorithms} and \citet{Acharya20} focus on settings where the privacy parameters vary across different neighboring datasets. Other models arise from altering the concept of neighboring datasets~\citep{kifer2014pufferfish, HeMD14}, or they emphasize indistinguishability guarantees within a redefined metric space for the datasets or data points~\cite{chatzikokolakis13, AlvimCPP18, imola2022balancing, andres2013geo}. Another relevant notion for classification tasks is {\em Label-DP} which considers labels of training data as sensitive, but  not the features themselves~\citep{chaudhuri2011sample, beimel2013private, wang2019sparse, ghazi2021deep}. 
A recent generalization of Label-DP, proposed by \citet{mahloujifar2023machine}, is Feature-DP, which allows specific features of the data to be disclosed by the algorithm while ensuring the privacy of the remaining information. A line of work, e.g., \cite{kifer2011no, Kenthapadi_Korolova_Mironov_Mishra_2013}, studies attribute DP, which limits changes in the output distribution of the privacy mechanism when a single coordinate of the data is modified.

A common shortcoming in the aforementioned approaches is that the heterogeneous privacy protection across features, without accounting for their correlations, can lead to unaddressed privacy leakage—an issue our framework is specifically designed to mitigate. Another key distinction of our work is the integration of the BCDP framework with the LDP framework, enabling simultaneous protection at both the feature level and the global level.

A Bayesian approach to Central-DP is considered in
\citet{triastcyn2020bayesian} where the notation of neighboring dataset are taken over pairs that are drawn from the same distribution. This is similar in spirit to our framework but we operate in the Local-DP regime so we work with a prior instead.
\citet{Xiao23} consider a Bayesian approach to privacy where the goal is to prevent an observer from being able to {\em nearly } guess the input to a mechanism under a given notion of distance and error probability.
Another Bayesian interpretation of DP is presented in \citet{kasiviswanathan2014semantics} and show that DP ensures that the posterior distribution after observing the output of a DP algorithm is close to the prior in Total-Variation distance. None of these results explicitly focus on varying privacy features across coordinates.

Private mean estimation and private least-squares regression are well studied problems in literature in both central \citep{biswas2020coinpress, karwa2017finite, Kifer12,Vu09} and local \citep{Asi22,Feldman2021,zheng2017collect} DP regimes.
We depart from these works as we consider a BCDP constraint in addition to an LDP constraint.

\paragraph{Organization:} We present a formal treatment of our framework in \cref{sec:formulation}. 
Some key properties of the framework are presented in \cref{sec:properties}. We consider the problem of {\em mean estimation} in \cref{sec:mean} and {\em ordinary least-squares regression} in \cref{sec:LSE}.

\section{Problem Formulation} \label{sec:formulation}
Let $\cX$ denote the data domain represented as the  Cartesian product $\cX := \bigtimes_{j=1}^d \cX_j$. One can consider this as a general decomposition where each of $d$ canonical data coordinates can be categorical or a real value, or a combination of both, as the setting dictates. For a vector $\vecx \in \cX$, we let $x_i$  denote the $i$-th coordinate, and $\vecx_{-i}$ the vector with $(d-1)$ coordinates  obtained by removing the $i$-th coordinate $x_i$ from $\vecx$. A private mechanism $M$ is a randomized algorithm that maps $\vecx$ to $M(\vecx)$ in an output space $\cY$.

Let $(\cX,\cF_{\cX})$ and $(\cY,\cF_{\cY})$ be measurable spaces. Let each coordinate space $\cX_i$ be equipped with  $\sigma$-algebra $\cF_{\cX_i}$, and take $\cF_{\cX}$ equal to the product $\sigma$-algebra $\bigtimes_{i=1}^d \cF_{\cX_i}$.
A randomized mechanism $M$ from $\cX$ to $\cY$ can be represented by a Markov kernel $\mu:(\cX, \cF_{\cX})\to (\cY,\cF_{\cY})$ where for an input $\vecx \in \cX$, the output $M(\vecx)$ is a $\cY$-valued random variable with law $\mu_{\vecx}(\cdot)$.

\paragraph{Local differential privacy} We first revisit the definition of Local Differential Privacy (LDP).
\begin{definition}[Local DP~\citep{kasiviswanathan2011can}] \label{definition:LDP}
A mechanism $M$ is $\varepsilon$-LDP if, for all $R \in \cF_{\cY}$  and all $\vecx, \vecx’ \in \cX$, we have:
\begin{equation}
\mu_{\vecx}(R) \leq e^\varepsilon \mu_{\vecx'}(R).   
\end{equation}
\end{definition}

With the understanding that $M(\vecx) \sim \mu_{\vecx}(\cdot)$, \cref{definition:LDP} can be rewritten in the more familiar form $\text{Pr}\{M(\vecx) \in R\} \leq e^{\epsilon} \text{Pr}\{M(\vecx') \in R\} \ \forall \vecx, \vecx' \in \cX, R \in \cF_{\cY}$.

Next, we focus on a {\em Bayesian interpretation} of this setting. 
Suppose we have an underlying data distribution $\pi$. That is, we equip $(\cX,\cF_{\cX})$ with the probability measure $\pi$ which we call the {\em prior}.
This induces a probability space $(\cX \times \cY, \cF_{\cX}\times \cF_{\cY}, P)$ where the measure $P$ is characterized via
\begin{equation}
    P(S \times R) = \int_S \mu_{\vecx}(R)d\pi(\vecx), \ S \in \cF_{\cX}, R \in \cF_{\cY}.
\end{equation}
Note that the mechanism $M$ defining $\mu$ is precisely represented on this space by the random variable $M: (\vecx,y) \in \cX\times \cY \mapsto y$.  By a slight abuse of notation, we continue to denote the (randomized) output of the mechanism by $M(\vecx)$.
We first define the sets on $\cX$ which have positive probability under $\pi$
\begin{align}
    \cF_{\cX}^+ &= \{ S \in \cF_{\cX}| \pi(S) > 0 \}.
\end{align}
For $S \in \cF_{\cX}^+$ and $R \in \cF_{\cY}$, note that
\begin{align}
P\{M(\vecx) \in R| \vecx \in S\} 
&= \frac{P(S \times R)}{\pi(S)}.
\end{align}

\paragraph{Bayesian differential privacy:} We now introduce a Bayesian formulation of the LDP definition which involves both the mechanism $M$ \emph{and} the prior $\pi$.
This can be seen as the local version of the definitions proposed in earlier works, e.g., \cite{kifer2014pufferfish, triastcyn2020bayesian}.

\begin{definition}[Bayesian DP] \label{def:pi-DP}
The pair $(\pi,M)$ is $\epsilon$-BDP if for all $R \in \cF_{\cY}$ and all $S,S' \in \cF_{
\cX}^+$, we have:
\begin{align} 
P\{M(\vecx) \in R| \vecx \in S\} \leq e^{\epsilon} P\{M(\vecx) \in R| \vecx \in S'\}. \label{eq:D-def-BDP}
\end{align}
\end{definition}

We show in Proposition~\ref{lemma:LDP_Bayesian} that $\epsilon$-LDP and $\epsilon$-BDP can be viewed as equivalent, modulo some technicalities.

\begin{restatable}{proposition}{lemLDPBayesian}
\label{lemma:LDP_Bayesian} 
If $M$ is $\epsilon$-LDP then $(\pi,M)$ is $\epsilon$-BDP.
Conversely, if $(\pi,M)$ is $\epsilon$-BDP, then (assuming $\cY$ is a Polish space) there exists a mechanism $M'$ which is $\epsilon$-LDP such that $P\{M(\vecx) = M'(\vecx)\} = 1$.
\end{restatable}
The reverse direction of the equivalence above relies on  $\cY$ being a Polish space, which should capture all settings of practical interest. The proof can be found in Appendix~\ref{sec:proof_of_lemma:LDP_Bayesian}.

As discussed in \cite{kifer2014pufferfish}, we can interpret the BDP definition in terms of odds ratios. The \textit{prior odds} of two disjoint events $S$ and $S'$ is given by $\pi(S)/\pi(S')$. Now, for a set $R \in \mathcal{F}_{\mathcal{Y}}$ with positive measure, the \textit{posterior odds}, after observing mechanism output $M(\vecx) \in R$, is equal to
\begin{equation}
\frac{P\{\vecx \in S| M(\vecx) \in R\}}{P\{\vecx \in S'| M(\vecx) \in R\}}.   
\end{equation}
One can easily verify that the BDP definition implies that the ratio of posterior and prior odds, also known as the \textit{odds ratio}, is in the interval $[e^{-\epsilon}, e^\epsilon]$. In other words, the BDP definition limits how much the output of the mechanism changes the odds, which in turn limits our ability to distinguish whether the data $\vecx$ was in $S$ or $S'$.

\paragraph{Coordinate differential privacy:} 
We wish to measure the privacy loss of each coordinate incurred by a mechanism $M$.
Initially, one might hypothesize that by altering the definition of LDP to reflect changes in specific coordinates, we could achieve privacy protection for certain features of the user's data. 
More precisely, consider the following definition of Coordinate DP (CDP)\footnote{The acronym CDP is not to be confused with Concentrated DP, which is not considered in this work.} where $\bm{c}$ represents a vector of the desired levels of privacy for each coordinate.

\begin{definition}[Coordinate DP] \label{def:coordinate-dp}
Let $\bm{c} \in \bbR_{\geq 0}^d$.
A mechanism $M$ is $\bm{c}$-CDP if, for all $R \in \cF_{\cY}$ and all $\vecx, \vecx' \in \cX$, we have the implication:
\begin{equation}
\vecx_{-i}=\vecx'_{-i} ~\Rightarrow~\mu_{\vecx}(R) \leq e^{c_i} \mu_{\vecx'}(R).
\end{equation}
\end{definition}
A special case of CDP, where all $c_i$’s are equal, is referred to as attribute DP in the literature \cite{kifer2011no, Kenthapadi_Korolova_Mironov_Mishra_2013, ghazi2022algorithms}. Additionally, CDP itself can be viewed as a special case of Partial DP, introduced in \cite{ghazi2022algorithms}.

The drawback of this definition is that it does not account for potential  correlation among the data coordinates and, therefore, does not provide the type of guarantee we hope for regarding sensitive parts of the data. In other words, in $\bm{c}$-CDP, stating that the privacy level of coordinate $i$ is $c_i$ could be misleading. For instance, having $c_i=0$ does not necessarily ensure that no information is leaked about $x_i$, as the other (potentially less-well protected) coordinates $\vecx_{-i}$ may be correlated with it.

\paragraph{Our formulation: Bayesian coordinate differential privacy:} In order to capture the correlation among the data coordinates, we adopt a Bayesian interpretation of differential privacy. 
As seen in \cref{lemma:LDP_Bayesian}, BDP and LDP are closely related and this serves as the main motivation for our definition.

For $i\in [d]$, define 
\begin{equation}
    \cF_{\cX_i}^+ = \{ S_i \in \cF_{\cX_i} | \pi\{x_i \in S_i\}  > 0 \}.
\end{equation}
For $S_i \in \cF_{\cX_i}^+$, we have
\begin{equation}
    P\{M(\vecx) \in R|x_i \in S_i\} = \frac{P( (S_i \times \cX_{-i}) \times R)}{\pi(S_i \times \cX_{-i})} ,  \label{eq:P-i-def}
\end{equation}
where $\cX_{-i} = \bigtimes_{j \in [d], j \neq i} \cX_j$.
Thus, similar to \cref{def:pi-DP}, we have a natural version of privacy for each coordinate which we shall refer to as Bayesian Coordinate DP (BCDP).

\begin{definition}[Bayesian Coordinate DP] \label{def:BCDP}
Let $\bm\delta \in \bbR_{\geq 0}^d$. 
A pair $(\pi,M)$ is $\bm\delta$-BCDP if, for each $i \in [d]$, and for all $ R \in \cF_{\cY}$, $S_i, S_i' \in \cF_{\cX_i}^+$, we have: 
\begin{equation} \label{eqn:def_BCDP}
P\{M(\vecx) \in R|x_i \in S_i\}  \leq e^{\delta_i} P\{M(\vecx) \in R|x_i \in S_i'\}.
\end{equation}
\end{definition}
Unlike the Coordinate-DP definition, this definition accounts for potential  correlation between coordinates through incorporation of  the prior $\pi$. 

It is worth noting that, similar to the BDP definition, we can interpret the BCDP definition in terms of the odds ratio, with the difference that we consider the two events $\{x_i \in S_i\}$ and $\{x_i \in S_i’\}$. In other words, having a small $\delta_i$ guarantees that the mechanism’s output does not provide much useful information for discriminating between the possibilities that the underlying data has coordinate $x_i$ in  $S_i$ or $S_i’$. This aligns with our goal of proposing a new formulation that characterizes how much information is revealed about each particular coordinate. We further discuss this interpretation in the next section.

\section{Properties of BCDP} \label{sec:properties}
In this section, we discuss a number of properties following from the BCDP definition. 
\cref{fig:DP-relations} compiles some of the results from \cref{sec:formulation} and \cref{sec:properties} on how the different notions of Local DP (LDP), Coordinate-DP (CDP), Bayesian DP (BDP), and Bayesian Coordinate-DP (BCDP) are interrelated.

\subsection{BCDP through the Lens of Hypothesis Testing} \label{sec:HT}
In this subsection, we present an interpretation of what the privacy vector $\bm\delta$ quantifies in  $\bm\delta$-BCDP.
We begin by revisiting the connection between hypothesis testing and the definition of differential privacy, and demonstrate how this connection extends to the $\bm{\delta}$-BCDP definition. For an $\epsilon$-LDP mechanism $M$ represented by Markov kernel $\mu$, suppose an adversary observes a realization of the output $Y$ and is tasked with determining whether the input was $\vecx$ or $\vecx’$. This scenario can be formulated as the following binary hypothesis problem:
\begin{align*}
H_0 &:  Y \sim \mu_{\vecx}(\cdot),\\
H_1 &: Y \sim \mu_{\vecx'}(\cdot).
\end{align*}

Let the rejection region be $R \subseteq \cY$.
Denote the type~I error (false alarm) rate as $\alpha(R,M,\vecx,\vecx')$ and the type~II error (missed detection) rate as $\beta(R,M,\vecx,\vecx')$.
We restate a known result below \citep{Kairouz15,Wasserman10}.

\begin{proposition} \label{lemma:HT-DP}
A mechanism $M(\cdot)$ is $\epsilon$-LDP iff the following condition holds for any $\vecx,\vecx' \in \cX$ and any $R \in \cF_{\cY}$:
\begin{align}    e^{\epsilon}\alpha(R,M,\vecx,\vecx') + \beta(R,M,\vecx,\vecx') &\geq 1. \label{eq:TypeI-II}
\end{align}
\end{proposition}
We note that the criterion \eqref{eq:TypeI-II} can be equivalently replaced by the criterion
$$
    \alpha(R,M,\vecx,\vecx') + e^{\epsilon} \beta(R,M,\vecx,\vecx') \geq 1.
$$

As noted in \cite{Kairouz15}, this operational perspective shows that it is impossible to  force type I and type II error rates to simultaneously vanish for a DP algorithm, and the reverse implication also holds. 

\paragraph{For BCDP:} 
Now, consider a $\bm\delta$-BCDP mechanism $M$. Suppose an adversary has knowledge of the prior $\pi$ and again observes the algorithm’s output $Y$. For any $i \in [d]$, the adversary aims to distinguish whether the output was generated from data where the $i$-th coordinate belongs to the set $S_i$ or $S_i’$. In other words, we consider the following binary hypothesis testing problem:
\begin{align*}
     H_0^i &:  Y \sim P\{\,\cdot\,|x_i \in S_i\},\\
     H_1^i &: Y \sim P\{\,\cdot\,|x_i \in S_i'\}.
\end{align*}
Let the rejection region for the $i$-th hypothesis test be $R_i \in \cF_{\cY}$.
For the $i$-th test, denote the type I and type II error rates as $\alpha(R_i,M,S_i,S_i')$ and $\beta(R_i,M,S_i,S_i')$.
Then, the following result holds:
\begin{restatable}{theorem}{thmHT}\label{theorem:HT}
A pair $(\pi,M)$ is $\bm\delta$-BCDP iff the following condition holds  for every coordinate $i \in [d]$:  
\begin{align}\label{eq:20-bcdp}
    e^{\delta_i}\alpha(R_i,M,S_i,S_i') + \beta(R_i,M,S_i,S_i') &\geq 1,
\end{align}
 for any $ S_i,S_i' \in \cF_{\cX_i}^+$ and $ R_i \in \cF_{\cY}$.
\end{restatable}
We note that the criteria \eqref{eq:20-bcdp} can be equivalently replaced by the criteria
$$\alpha(R_i,M,S_i,S_i') + e^{\delta_i}\beta(R_i,M,S_i,S_i') \geq 1.$$
The proof of \cref{theorem:HT} is presented in \cref{sec:HT-proof}.
Thus, the theorem shows that the vector $\bm\delta$ in the $\bm\delta$-BCDP formulation captures the privacy loss for each of the coordinates in the same way $\epsilon$ in $\epsilon$-DP captures the privacy loss of the entire data.
In contrast, the vector $\vecc$ in $\vecc$-CDP does not have such an interpretation of privacy loss across the coordinates.

\begin{figure}
    \centering
\begin{tikzpicture}[
  node distance=3cm and 4cm,
  box/.style={draw, minimum width=2cm, minimum height=1cm, align=center},
  arrow/.style={-{Latex[length=3mm, width=2mm]}}
]

\node[box] (DP) {$\epsilon_L$-LDP};
\node[box, right=of DP] (BDP) {$\epsilon_B$-BDP};
\node[box, below=of DP] (CDP) {$\bm c$-CDP};
\node[box, right=of CDP] (BCDP) {$\bm\delta$-BCDP};

\draw[arrow] (DP) -- node[above] {$\epsilon_L = \epsilon_B$} (BDP);
\draw[arrow] (BDP) -- (DP);
\draw[arrow] (BDP)  -- node[right] {$\bm\delta \preceq \epsilon_B \bm1$ } (BCDP);
\draw[arrow] (DP.225) -- node[left] { $\bm c \preceq \epsilon_L \bm1$} (CDP.135);
\draw[arrow] (CDP) -- node[above] { $\bm \delta \preceq \text{func. of }(\bm c, \pi)$} (BCDP);
\draw[arrow] (CDP.45)  -- node[right] { $\epsilon_L \leq  \sum_i c_i$} (DP.315);
\end{tikzpicture}
    \caption{General relation of LDP (\cref{definition:LDP}), BDP (\cref{def:pi-DP}), CDP (\cref{def:coordinate-dp}) and BCDP (\cref{def:BCDP}). 
    The condition $\veca \preceq \vecb$ is to be read as $a_i \leq b_i \ \forall i$.
    The implication arrows and the described transformation of the parameters are to be interpreted as sufficient condition. 
    For example, an $\epsilon_L$-DP mechanism is guaranteed to be $\vecc$-CDP for $\vecc \preceq \epsilon_L \bm1$.  
    The function that translates $\vecc$-CDP under the prior $\pi$ to BCDP is presented in the \cref{prop:cdp-to-bdp}.} 
    \label{fig:DP-relations}
\end{figure}
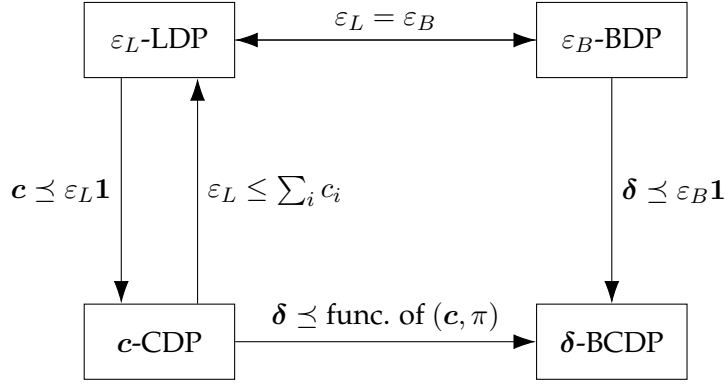

\subsection{BCDP can be achieved via LDP}
We begin by a simple result that follows directly from $\epsilon$-BDP definition.
If a mechanism is $\epsilon$-LDP, then it is also $\epsilon$-BDP.
Noticing that $\epsilon$-BDP ensures a ratio of $e^{\epsilon}$ for each coordinate as well in \eqref{eqn:def_BCDP}, we see that each coordinate has a privacy level of $\epsilon$ as well.
\cref{proposition:LDP_to_BCDP} formalizes this observation and the proof can be found in \cref{sec:ldp-to-bcdp-proof}.

\begin{restatable}{proposition}{propLdpToBcdp}{\normalfont{(LDP to BCDP)}}
\label{proposition:LDP_to_BCDP}
An $\varepsilon$-LDP mechanism satisfies $\varepsilon\bm1$-\bdp.
\end{restatable}
Thus, a baseline approach to implement $\bm\delta$-\bdp is to use a $(\min_i \delta_i)$-LDP mechanism. Nevertheless, this would mean that we provide the most conservative privacy guarantee required for one coordinate to all coordinates. However, as we will show, we can maintain an $\epsilon > \min \bm{\delta}$ for the overall privacy guarantee while providing the $\bm{\delta}$-\bdp guarantee, therefore achieving lower error for downstream tasks.

\subsection{BCDP does not imply LDP}
It should be noted that BCDP does not ensure LDP, BDP, or CDP.
Example~\ref{ex:bcdp-ldp} below demonstrates a mechanism that is BCDP, but  not  LDP.

\begin{example}
\label{ex:bcdp-ldp}
Consider data $\vecx = (x_1,x_2)$ where $x_1, x_2$ are i.i.d.~Bern$(\frac{1}{2})$, and let $\mathcal{Y} = \{0,1\}$.
Let $M$ be defined by the table below, with parameters $a,b,c \in (0,1]$.
Mechanism $M$ is not LDP (and hence, neither BDP nor CDP).
However, for BCDP, the terms $\Pr\{M(\vecx) = y|x_i = 0 \}$ and $\Pr\{M(\vecx) = y|x_i = 1 \}$ have a finite multiplicative ranges in $y=0,1$, for both  $i=1,2$.
For instance, if $a=b=c=1/2$, the resulting $M$ is  $(\log(2),\log(2))$-BCDP.
\renewcommand{\arraystretch}{1.8}
\begin{table}[H]
\centering
\begin{tabular}{|cc|cc|}
\hline
\multicolumn{2}{|c|}{\multirow{2}{*}{$P\{M(\vecx)=1|\vecx\}$}}      & \multicolumn{2}{c|}{$x_1$}                                 \\ \cline{3-4} 
\multicolumn{2}{|c|}{}                       & \multicolumn{1}{c|}{0}                   & 1           \\ \hline
\multicolumn{1}{|c|}{\multirow{2}{*}{$x_2$}} & 0 & \multicolumn{1}{c|}{$a$} & $b$ \\ \cline{2-4} 
\multicolumn{1}{|c|}{}                   & 1 & \multicolumn{1}{c|}{$0$}         & $c$ \\ \hline
\end{tabular}
\end{table}
\end{example}
The takeaway from the above example is that there can be priors and mechanisms such that the prior-mechanism pair is $\bm\delta$-\bdp but not $\varepsilon$-DP for any value of $\varepsilon < \infty$. 
Intuitively, BCDP quantifies privacy loss for each coordinate.
In \cref{ex:bcdp-ldp}, one cannot distinguish much about any particular coordinate when the output is observed due to the BCDP guarantee.
However, one may still be able to make conclusions over the coordinates jointly.
For instance,  if the mechanism output in Example \ref{ex:bcdp-ldp} is $1$, then the input could not have been $\vecx =(0,1)$.

As this result suggests, just ensuring \bdp is not sufficient to guarantee overall privacy. 
In other words, the right way to think of \bdp is that it is a tool that allows us to quantify the privacy offered by an algorithm for each coordinate. The user can demand overall $\epsilon$-LDP along with $\bm{\delta}$-\bdp where the value $\delta_i$ could be significantly smaller than $\epsilon$ for those coordinates that the user feels more concerned about their privacy. 

\subsection{Achieving BCDP via CDP}\label{sec:CDPTOBCDP}
As we discussed earlier, CDP does not imply the \bdp, as it does not take into account the correlation among the features. Therefore, a natural question arises: could CDP, along with some bound on the dependence among the features, lead to a bound with respect to the \bdp definition? The next result provides an answer to this question. We denote the total variation distance between two distributions $P$ and $Q$ as $\mathsf{TV}(P,Q)$.

\begin{restatable}{proposition}{propCdpToBdp}{\normalfont{(CDP to BCDP)}} \label{prop:cdp-to-bdp}
Suppose mechanism $M$ is $\bm{c}$-CDP
and there exists $q_1,\ldots,q_d$ such that
\begin{equation} \label{eq:tv-bound}
    \mathsf{TV}(\pi_{\vecx_{-i}|x_i \in S_i},\pi_{\vecx_{-i}|x_i \in S_i'}) \leq q_i \ \forall \ S_i,S_i' \in \cF_{\cX_i}^+.
\end{equation}
Then: 
\begin{enumerate} 
\item The mechanism is $\sum_{i=1}^n c_i$-LDP and $\bm\delta$-\bdp with $\delta_i = c_i + \log (1 + q_i e^{\sum_{j\neq i} c_j} - q_i)$. 
\item In case that the mechanism is $\varepsilon$-LDP for some $\varepsilon \geq 0$, then it is $\bm\delta'$-\bdp with $\delta'_i = \min \{c_i + \log (1 + q_i e^{\varepsilon} - q_i), \epsilon\}$.    
\end{enumerate} 
\end{restatable}
For the proof, see Appendix~\ref{sec:prop:cdp-to-bdp}. 
To better highlight this result, consider a simple example in which that the input's $d$ coordinates are redundant copies of the same value and we have access to a $\vecc$-CDP mechanism $M$. Then, as \cref{prop:cdp-to-bdp} suggests, and given that we have $q_i=1$ here, this mechanism is $\bm{\delta}$-BCDP with $\delta_i = \sum_{j} c_j$. This is, of course, intuitive, as all coordinates reveal information about the same value. It also reinforces that the BCDP is sensitive to correlations, unlike CDP.
The first part of \cref{prop:cdp-to-bdp} provides an immediate method for constructing a BCDP mechanism using mechanisms that offer CDP guarantees, such as applying off-the-shelf LDP mechanisms per coordinate. This can also be done without requiring full knowledge of the prior.
The second part of \cref{prop:cdp-to-bdp} suggests that in cases where we design an algorithm that has a tighter LDP guarantee than just $\sum_{i=1}^n c_i$, we can further improve the BCDP bound. One example of such a construction is provided in our private mean estimation algorithm in \cref{sec:mean}.
One challenge with this approach is translating the constraint $\bm\delta$ to a vector $\bm c$. In \cref{sec:proof:lemma:CDP_to_BCP_approximation}, we provide a tractable approach to do so by imposing a linear constraint on $\vecc$.

\subsection{Do Composition and Post-Prossesing Hold for BCDP?}
Composition is a widely used and well-known property of LDP. Nonetheless, as the following example highlights, the composition, as traditionally applied in differential privacy settings, does not hold in the BCDP setting.
\begin{example}
Let $\vecx \in \mathbb{R}^3$ be a Bernoulli vector where coordinates' distributions are independent and each one is $\text{Ber}(1/2)$. Consider the mechanism 
\begin{equation}
M(\vecx) = \begin{cases}
[x_1 \oplus x_2, x_3]^\top & \text{ with probability } \frac{1}{2} \\
[x_2, x_1 \oplus x_3]^\top & \text{ with probability } \frac{1}{2},
\end{cases}    
\end{equation}
where $\oplus$ denotes the sum in $\mathbb{Z}_2$. It is straightforward to verify that $M(\cdot)$ is in fact $\delta_1=0$-\bdp with respect to the first coordinate. However, having two independent copies of $M(\cdot)$ is not \bdp with respect to the first coordinate. To see this, consider the case where $x_1=1$ and, in two copies of $M(\vecx)$, one adds $x_1$ to $x_2$ and the other adds it to $x_3$. 
By observing the two copies, we can determine $x_1=1$.
\end{example}
On the other hand, the post-processing property holds for BCDP (see Appendix~\ref{sec:PostProc} for the proof).
\begin{restatable}{proposition}{lemPostProc}{\normalfont{(Post processing for BCDP)}}
\label{lemma:bdp_data_processing}
Let $M(\cdot)$ be a $\bm \delta$-\bdp mechanism and $K:\mathcal{Y} \to \mathcal{Z}$ be a measurable function. Then, $K\circ M$ is also $\bm \delta$-\bdp.  
\end{restatable}

\section{Private Mean Estimation} \label{sec:mean}


\floatname{algorithm}{Mechanism}
\begin{algorithm}[!t]
   \caption{Proposed locally private mechanism $M_{\mathsf{mean}}(\vecx,\vecc)$}
   \label{alg:local-channels}
\begin{algorithmic}
   \STATE {\bfseries Input:} user data $\vecx \in [-1,1]^d$ and  non-decreasing vector $\vecc$.
   \STATE $c_0 \gets 0$
   \STATE For $k \in [d]$, set $w_k =  \frac{(c_k-c_{k-1})^2}{(d-k+1)}$
   \FOR{$k \in [d]$}{} 
   \STATE $\vecY^{k} \gets M_{\mathsf{LDP}}(\vecx_{k:d}, c_k - c_{k-1}, \sqrt{d-k+1}) \in \bbR^{d-k+1}$
   \ENDFOR
   \STATE  $\hat{\nu}_i \gets  (\sum_{k=1}^i  Y_{i+1-k}^k w_k)/(\sum_{k=1}^i w_k)$ for $i \in [d]$
   \STATE Return $\hat{\bm{\nu}}$
\end{algorithmic}
\end{algorithm}

\begin{figure}[!t]
    \centering
    \begin{tikzpicture}
    \node (x1) at (0, 4) {$x_1$};
    \node (x2) at (1.5, 4) {$x_2$};
    \node (dots) at (3, 4) {$\dots$};
    \node (xd1) at (4.5, 4) {$x_{d-1}$};
    \node (xd) at (6, 4) {$x_d$};

    \draw (6.5, 3.9) -- (6.5, 3.75) -- (5.5, 3.75) -- (5.5, 3.9);
    \draw[->] (6, 3.75) -- (6, 1.5) node[pos=0.7,right,red] {$c_d - c_{d-1}$};
    \draw (6.5, 1.35) -- (6.5, 1.5) -- (5.5, 1.5) -- (5.5, 1.35);
    
    \node[violet] (y1_d) at (6, 1.15) {$Y_1^{d}$};

    \draw (6.5, 3.65) -- (6.5, 3.5) -- (4, 3.5) -- (4, 3.9);
    \draw[->] (5.25, 3.5) -- (5.25, 0.75) node[midway,left,red] {$c_{d-1}-c_{d-2}$};
    \draw (6.5, 0.6) -- (6.5, 0.75) -- (3.75, 0.75) -- (3.75, 0.6);
    
    \node[violet] (y1_d1) at (4.5, 0.4) {$Y_1^{d-1}$};
    \node[violet] (y2_d1) at (6, 0.4) {$Y_2^{d-1}$};

    \node at (4.5, -0.1) {$\vdots$};
    \node at (4.5, 3.3) {$\vdots$};

    \draw (6.5, 3) -- (6.5, 2.85) -- (-0.5, 2.85) -- (-0.5, 3.9);
    \draw[->] (3, 2.85) -- (3, -0.6) node[midway,left,red] {$c_{1}-c_{0}$};
    \draw (6.5, -0.75) -- (6.5, -0.6) -- (-0.5, -0.6) -- (-0.5, -0.75);

    \node[violet] (y1_1) at (0, -0.95) {$Y_1^{1}$};
    \node[violet] (y2_1) at (1.5, -0.95) {$Y_2^{1}$};
    \node (dots) at (3, -0.95) {$\dots$};
    \node[violet] (yd1_1) at (4.5, -0.95) {$Y_{d-1}^{1}$};
    \node[violet] (yd_1) at (6, -0.95) {$Y_d^{1}$};




\end{tikzpicture}

    \caption{The figure illustrates how the mechanism $M_{\mathsf{mean}}$ in Mechanism \ref{alg:local-channels} obtains the estimate $\bm\hat{\nu}$. For example, the vector \violet{$Y^k$} is obtained by taking the vector $(x_{k},x_{k+1},\ldots,x_{d})$ and using the LDP channel $M_{\mathsf{LDP}}$ with privacy parameter \red{$c_k - c_{k-1}$}.
    $\hat{\nu}_i$ is obtained by taking a linear combination of the component of \violet{$\{Y^k\}_{k \in [d]}$} corresponding to $x_i$, i.e., directly below $x_i$ in the figure.}
    \label{fig:algorithm-scheme}
\end{figure}
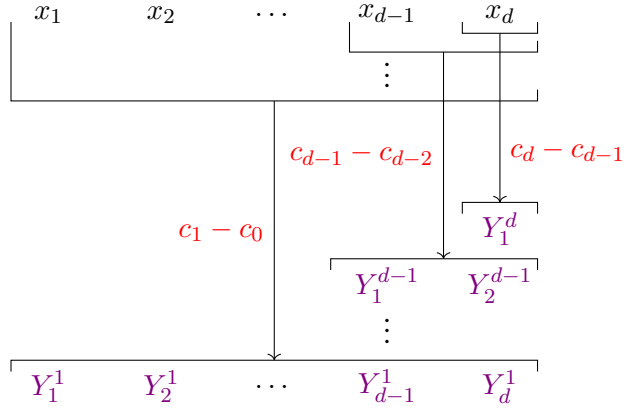

\floatname{algorithm}{Mechanism}
    \begin{algorithm}[t]
   \caption{Local DP channel for data in $\cB_2(r)$ proposed by \citet{Duchi13}}
   \label{alg:duchi-mldp}
\begin{algorithmic}
   \STATE {\bfseries Input:} user data $\vecv$, Local-DP level $\alpha \geq 0$, and bound $r$ such that $\vecv \in \cB_2(r)$. 
   \STATE $d \gets $dimension($\vecv$)
   \STATE $B \gets \sqrt{\pi}dr \frac{e^{\alpha} + 1}{e^{\alpha} - 1}\frac{ \Gamma(\frac{d+1}{2})}{\Gamma(\frac{d}{2}+1)}$
   \STATE $K \sim $ Bern$(\frac{e^{\alpha}}{e^{\alpha} + 1})$
   \STATE $S \sim $ Bern$(\frac{1}{2} + \frac{\|\vecv\|_2}{2r})$
   \vspace{4pt}\STATE $\tilde{\vecv} \gets (2S-1) \frac{r\vecv}{\|\vecv\|_2}$
   \STATE Return $Z \sim $ Uniform$(\vecz \in \bbR^d: (2K-1)\vecz^T \tilde{\vecv} > 0, \|\vecz\|_2 = B)$
\end{algorithmic}
\end{algorithm}


To demonstrate the practical use of the BCDP framework, we consider the private mean estimation problem.
\paragraph{Problem Setup:} Let there be $n$ users submitting their data to a central server via Local DP mechanisms. 
Let user $j$'s data  be $\vecx^{(j)} \in [-1,1]^d$, where $\vecx^{(j)}$ is drawn from prior $\pi^{(j)}$ on $[-1,1]^d$;
the server and users are not assumed to know the  priors (which may be different for each individual user).
Our goal is to estimate the empirical mean $\bar{\vecx} = \sum_{j=1}^n \frac{\vecx^{(j)}}{n} \in [-1,1]^d$ under $\varepsilon$-LDP and $\bm\delta$-\bdp.
We consider $l_2$-norm as our error metric. 
Without loss of generality, we assume that $\bm\delta$ is in non-decreasing order. 
We make the following assumption regarding the priors.
\begin{assumption} \label{assumption:bounded_TV}
For every $i \in [d-1]$ and every $j \in [n]$, we have 
 \begin{equation} \label{eq:tv-bound-universal}
\mathsf{TV}(\pi^{(j)}_{x_{-i}|x_i \in S_i},\pi^{(j)}_{x_{-i}|x_i \in S_i'}) \leq q \quad \forall S_i,S_i' \in \cF_{\cX_i}^+.
\end{equation}
\end{assumption}
Here $q$ can be interpreted as a measure of how much different coordinates are interdependent. For instance, in the case where coordinates are independent we have $q=0$, and in the case where they are fully equal, we have $q=1$. Note that we do not need \eqref{eq:tv-bound-universal} to hold for the last coordinate ($i=d$). While this slight relaxation does not play a role in this section, it will be useful when we move to the private optimization application.

It is also worth emphasizing that our analysis extends to cases where we have heterogeneous bounds across different coordinates (similar to $ \eqref{eq:tv-bound} $) and different across users. 
However, we opt for a universal bound for the sake of simplifying the presentation of results.
We assume that the upper bound $q$ is known to the designer of the local channels, which can be either the end user themselves or the central server.

Our proposed privacy mechanism $M_{\mathsf{mean}}$ (Mechanism \ref{alg:local-channels}) builds upon any LDP mechanism $M_{\mathsf{LDP}}$, capable of handling different dimensional inputs, that can be used for mean estimation as a black-box tool.
While \cref{alg:local-channels} satisfies the privacy constraints for any black-box LDP mechanism $M_{\mathsf{LDP}}$, we shall focus on the mechanism presented by \citet{Duchi13} for proving error bounds.
The mechanism of \citet{Duchi13}, which is known to be minimax optimal for mean estimation, is outlined in Mechanism \ref{alg:duchi-mldp} for completeness.
A figure explaining $M_{\mathsf{mean}}$ is presented in Figure~\ref{fig:algorithm-scheme}.

\cref{theorem:mean_estimation} shows that our proposed mechanism satisfies the desired privacy constraints and provides a bound on the mean-squared error (MSE).
The notation  $a \wedge b$ refers to $\min\{a,b\}$.
later in this section we present a refined result with a more interpretable MSE bound in \cref{cor:mse}.

\begin{restatable}{theorem}{thmMeanEstimation}\label{theorem:mean_estimation}
Suppose \cref{assumption:bounded_TV} holds and let $\tilde{\delta}_i:= \min\{\delta_i, \varepsilon\}$ for any $i \in [d]$. Then, $M_{\mathsf{mean}}$, i.e., Mechanism \ref{alg:local-channels}, with parameters
\begin{equation} \label{eq:c-values}
c_d = \min \left \{  \log(\frac{e^{\zeta \tilde{\delta}_1}+q-1}{q}), \tilde{\delta}_d \right \}, \, c_0 = 0, 
\end{equation}
\begin{equation*}
\forall i < d, \ c_{i} = \begin{cases}
    c_d  & \text{if }c_d \leq \tilde{\delta}_i,  \\
    \tilde{\delta}_i - \log(1+qe^{c_d}-q) & \text{otherwise},
\end{cases}
\end{equation*}
where $\zeta \in (0,1]$ is a free parameter, is $\bm\delta$-\bdp and $\varepsilon$-LDP. 
Moreover, using Mechanism \ref{alg:duchi-mldp} as $M_{\mathsf{LDP}}$, we obtain the following mean-squared error bound
\begin{equation} \label{eqn:MSE}
 \bbE \left[  \bigg \| \bar{\vecx}  - \Pi \bigg ( \frac{1}{n}\sum_{j=1}^n M_{\mathsf{mean}}(\vecx^{(j)},\vecc) \bigg)  \bigg \|_2^2  \bigg| \vecx^{(1)},\ldots,\vecx^{(n)} \right]   
\lesssim \frac{1}{n} \sum_{i=1}^d \left ( \frac{1}{\sum_{k=1}^i \frac{(c_k - c_{k-1})^2}{d-k+1}} \right ) \wedge d ,    
\end{equation}
where $\Pi(\cdot)$ denotes projection into $[-1,1]^d$ and
the expectation is taken over the randomness in $M_{\mathsf{mean}}(\cdot)$.
\end{restatable}
The proof is presented in Appendix~\ref{sec:proof:theorem:mean_estimation}. Here, we provide a proof sketch for the privacy guarantee. We first establish that $M_{\mathsf{mean}}$ is $\vecc$-CDP and $c_d$-LDP. The latter implies that $M_{\mathsf{mean}}$ also satisfies $c_d \bm1$-BCDP. Thus, if $\tilde\delta_i \geq c_d$, we have already shown that $M_{\mathsf{mean}}$ is $\delta_i$-BCDP with respect to the $i$-th coordinate. For the case where $c_d > \tilde\delta_i$, we use \cref{prop:cdp-to-bdp}, which essentially implies that $M_{\mathsf{mean}}$ is $c_i + \log(1+qe^{c_d}-q)$-BCDP with respect to the $i$-th coordinate. Substituting the value of $c_i$ from the statement of \cref{theorem:mean_estimation} gives us the desired $\delta_i$-BCDP guarantee for coordinate $i$. Finally, the choice of $c_d$ in \eqref{eq:c-values} ensures that the proposed $c_i$’s are all non-negative and valid. 

We next make a few remarks regarding this result. \textit{First,} note that while this bound applies to empirical mean estimation, when the data points $\{\vecx^{(j)}\}_{j=1}^n$ are independent and identically distributed, the result can also be used to derive an out-of-sample error bound for the statistical private mean estimation problem. In this case, an additional $\frac{d}{n}$ term would be added to the error bound. 
\textit{Secondly,} note that in the special case where there is no coordinate-level privacy requirement (i.e., $\delta_i = \infty$), setting $\zeta = 1$ implies $c_1 = \ldots = c_d = \epsilon$, which yields the error bound $\mathcal{O}({d^2}/(n \varepsilon^2))$. This matches the $\varepsilon$-LDP lower bound provided in \citet{Duchi13}.
\textit{Lastly,} we would like to discuss the role of parameter $\zeta$. 
Note that when the correlation parameter $q$ is sufficiently high or the most sensitive privacy requirement $\delta_1$ is low enough such that $c_d$ is determined by the first term, $\log((e^{\zeta \tilde{\delta}_1} + q - 1)/{q})$, in \eqref{eq:c-values}, one can verify that $c_1 = (1 - \zeta)\delta_1$. As $\zeta$ increases, $c_1$ decreases, leading to a worse estimate for the first coordinate. However, $c_d$ increases, indicating that, for the less-sensitive coordinates, we are increasing the privacy parameter, potentially resulting in lower estimation error for those coordinates. In other words, the parameter $\zeta$ allows us to control how we allocate the privacy budget of the most sensitive coordinate, $\delta_1$, between the estimation of the first coordinate and the privacy cost imposed by its correlation with other coordinates. 

We next present a corollary that studies a special case where the coordinates are divided into more sensitive and less sensitive groups. In this case, users requests a general $\epsilon$-LDP guarantee for all coordinates and a $\delta$-BCDP guarantee specifically for the more sensitive coordinates. Under this scenario, we further simplify the error bound which allows us to better highlight implications of \cref{theorem:mean_estimation}.
\begin{restatable}{corollary}{CorMse} \label{cor:mse}
Suppose we are under the premise of \cref{theorem:mean_estimation}, with $\delta_i = \delta$ for $1 \leq i \leq k$ and $\delta_i = \epsilon > 2\delta$ for $d \geq i > k$. Then, by choosing $\zeta = 0.5$, the mean-squared error upper bound in \eqref{eqn:MSE} simplifies to
\begin{align}
\mathcal{O}(1) \frac{1}{n} \cdot \left ( \frac{dk}{\delta^2} + \frac{(d-k)^2}{\epsilon^2} \right ),
\end{align}
for the case $q \leq \frac{e^{\delta/2}-1}{e^\epsilon-1}$, and to
\begin{align}
\mathcal{O}(1) \frac{1}{n} \cdot \left ( \frac{dk}{\delta^2} + \frac{(d-k)^2}{ \frac{d-k}{d} \delta^2 + (c_d-\delta/2)^2} \right ),
\end{align}
otherwise, with $c_d = \log(\frac{e^{ \delta/2}+q-1}{q})$ decreasing from $\epsilon$ to $\delta/2$ as $q$ increases from $(e^{\delta/2}-1)/(e^\epsilon-1)$ to 1.
\end{restatable}
Under the special structure of $\bm\delta$ and $\epsilon$ in \cref{cor:mse}, the MSE can be thought of as sum of MSE of the more sensitive and the less sensitive coordinates.
For low levels of correlation, the MSE of less private coordinates behave like the MSE of an $\epsilon$-LDP mechanism on $d-k$ dimensional vector but the MSE of the more sensitive part \textit{depends on the dimension $d$ of the whole vector}, not just the sensitive part, showing how the more sensitive part of the data affects the whole vector.
As the correlation increases $q \to 1$, we have $c_d \to \delta/2$ and the MSE matches that of a $\min \bm\delta$-LDP mechanism. 
A supporting experiment is deferred to Appendix~\ref{sec:mean-exp}.
\section{Private Least-Squares Regression} \label{sec:LSE}
We focus on how the BCDP framework can be employed on least-squares problems. 
Consider a setting in which $n$ datapoints are distributed among $n$ users, where $\vecx^{(i)} = [{\vecz^{(i)}}^\top \ l^{(i)}]^\top \in \mathbb{R}^{d}$ represents the data of user $i$, with $\vecz^{(i)}$ as the feature vector for the $i$-th data point and $l^{(i)}$ as its label. Our goal is to solve the following empirical risk minimization problem privately over the data of $n$ users:
\begin{equation} \label{eqn:main_optimization}
\min_{\btheta \in \Theta} f(\btheta ) := \frac{1}{n} \sum_{i=1}^n \ell(\btheta, \vecx^{(i)}),
\end{equation}
where $\ell(\btheta, \vecx^{(i)}) := \frac{1}{2}(\btheta^\top \vecz^{(i)} - l^{(i)})^2$ and $\Theta \subset \mathbb{R}^{d-1}$ is a compact convex set. We denote the solution of \eqref{eqn:main_optimization} by $f^*$, achieved at some point $\btheta^*$.

The customary approach in private convex optimization is to use a gradient-based method, perturbing the gradient at each iteration to satisfy the privacy requirements. 
However, this approach is not suitable for cases in which coordinate-based privacy guarantees are needed. The difficulty is that each coordinate of the gradient is typically influenced by all  coordinates of the data.  As a result,  and it becomes challenging to distinguish the level of privacy offered to different coordinates of the data. Instead, we first perturb the users’ data locally and then compute the gradient at the perturbed point to update the model at the server. 

However, computing the gradient based on perturbed data introduces its own challenges. In particular, the gradient of $\ell(\btheta, \vecx)$ with $\vecx = [{\vecz}^\top l]^\top$, given by
\begin{equation} \label{eqn:gradient_ell}
\nabla \ell(\btheta, \vecx) = \vecz\vecz^\top\btheta - l \vecz,    
\end{equation} 
is not a linear function of the data. As a result, although the privacy mechanism outputs an unbiased estimate of the true data, this non-linearity introduces an error in the gradient with a non-zero mean, preventing the algorithm from converging to the optimal solution. 
To overcome this challenge, we observe that the aforementioned non-linearity in the data arises from the product terms $\vecz\vecz^\top$ and $l\vecz$. Therefore, if we create two private copies of the data and replace each term of the product with one of the copies, we obtain a stochastic unbiased estimate of the gradient at the server. The catch, of course, is that we need to divide our privacy budget in half. 

A summary of the algorithm is presented in \cref{alg:LSE_BCDP}. There,  $\text{OPT}(g, \delta)$ refers to any convex optimization algorithm of choice, e.g., gradient descent, that finds the minimizer of the function $g(\cdot)$ over the set $\Theta$ and up to an error of $\delta$.
\floatname{algorithm}{Algorithm}
\begin{algorithm}[t]
   \caption{BCDP and LDP least-squares regression}
   \label{alg:LSE_BCDP}
\begin{algorithmic}
   \STATE {\bfseries Input:} Users' data $\{\vecx^{(i)}\}_{i=1}^n$ and vector $\vecc$ 
   \FOR{$i \in [n]$}{}  
   \STATE {//user $i$ sends locally perturbed data to the server}
   \STATE $\bm{\hat{x}}^{(i,1)} \gets M_{\mathsf{mean}}(\vecx^{(i)},\vecc/2)$
   \STATE $\bm{\hat{x}}^{(i,2)} \gets M_{\mathsf{mean}}(\vecx^{(i)},\vecc/2)$
   \ENDFOR
   \STATE {//server optimizes with the perturbed data}
   \STATE $\hat{f}(\btheta) := \frac{1}{2n} \sum_{i=1}^n \left (\btheta^\top \bm{\hat{z}}^{(i,1)}\bm{\hat{z}}^{{(i,2)}^\top} \btheta - 2\btheta^\top  \hat{l}^{(i,1)} \bm{\hat{z}}^{(i,2)} \right )$
   \STATE $\bm{\hat{\theta}} \gets \text{OPT}(\hat{f}, \frac{1}{{n}})$ 
   \STATE  Return $\bm{\hat{\theta}}$
\end{algorithmic}
\end{algorithm}
We make the assumption that the data is bounded.
\begin{assumption} \label{assumption:bounded_data}
 $\vecx^{(i)} \in [-1,1]^d$ for all $i\in[n]$.   
\end{assumption}
We also need \cref{assumption:bounded_TV}, which bounds the dependence of each coordinate on the other coordinates of the data. 
Recall that the last coordinate is exempt from this assumption, which is why we position the label as the last coordinate of the data. It is not reasonable to assume the feature vector is weakly correlated with the label, as that would violate the premise of regression. We next state the following result on \cref{alg:LSE_BCDP}. 
\begin{restatable}{theorem}{thmLSE}\label{theorem:LSE}
Suppose Assumptions \ref{assumption:bounded_TV} and \ref{assumption:bounded_data} hold. Then, \cref{alg:LSE_BCDP}, with $\vecc$ set similar to \cref{theorem:mean_estimation},
is $\bm\delta$-\bdp and $\varepsilon$-LDP. 
Moreover, using Mechanism \ref{alg:duchi-mldp} as $M_{\mathsf{LDP}}$, we obtain the following error rate
\begin{equation}
 \bbE \left[ f(\bm{\hat{\theta}}) - f^*  \bigg| \vecx^{(1)},\ldots,\vecx^{(n)} \right] \lesssim 
\max_{\btheta \in \Theta} (\|\btheta\|^2+ \|\btheta\|) \left( \frac{\log d}{\sqrt{n}} (r^2 + rd) \wedge d \right ),
\end{equation}
with $r^2 = \sum_{i=1}^d 1/ ( \sum_{k=1}^i \frac{(c_k - c_{k-1})^2}{d-k+1}) $.
\end{restatable}
The proof of this theorem can be found in Appendix \ref{proof:TheoremLSE}. It is worth noting that $r^2$ appears in the mean-squared error bound in \cref{theorem:mean_estimation}. That said, if we consider special cases similar to the setting in \cref{cor:mse}, the term $r^2$  would simplify in a similar manner.

We conclude this section by noting that, while we focus on linear regression here, our approach can provide an unbiased gradient and guarantee privacy in any setting where the loss function \( \ell(\btheta, \vecx) \) has a quadratic (or even low-degree polynomial) dependence on \( \vecx \), such as neural networks with $\ell_2$  loss. However, the optimization error in such cases may require further considerations, as non-convex loss functions could arise.

\section{Acknowledgements}
This work was done in part while M.A. was a research fellow at the Simons Institute for the Theory of Computing.
A.F. and M.J. acknowledge support from the European Research Council (ERC-2022-SYG-OCEAN-101071601).  Views and opinions expressed are however those of the authors only and do not
necessarily reflect those of the European Union or the European Research Council
Executive Agency. Neither the European Union nor the granting authority can be
held responsible for them.
S.C. acknowledges support from AI Policy Hub, U.C. Berkeley; S.C. and T.C. acknowledge support from NSF CCF-1750430 and CCF-2007669.
\bibliographystyle{abbrvnat}
\bibliography{References}
\newpage
\appendix
\section{Proofs for Section 2 and Section 3}
\subsection{Proof of \cref{lemma:LDP_Bayesian}}\label{sec:proof_of_lemma:LDP_Bayesian}
\lemLDPBayesian*

\begin{proof}
(A) Forward direction:

For any $R \in \cF_{\cY}$ and $S,S' \in \cF_{\cX}^+$, we have
\begin{align}
    \mu_{\vecx}(R) &\leq e^{\epsilon} \mu_{\vecx'}(R) \\
    \implies \int_{\vecx \in S} \mu_{\vecx}(R) d\pi(\vecx) &\leq e^{\epsilon} \int_{\vecx \in S} \mu_{\vecx'}(R) d\pi(\vecx) \\
    \implies \Pr\{M(\vecx) \in R, \vecx \in S\} &\leq e^{\epsilon}  \pi(S) \mu_{\vecx'}(R) , \\
    \implies  \Pr\{M(\vecx) \in R| \vecx \in S\} &\leq e^{\epsilon} \mu_{\vecx'}(R). 
\end{align}
Similarly integrate again $\int_{\vecx' \in S'} \cdot d\pi(\vecx')$ to get the desired result.

 (B) Converse:

 Fix $R \in \cF_{\cY}$ and let 
 \begin{align}
     U(R) &= \mathsf{ess\ sup}_{\vecx \in \cX} \mu_{\vecx}(R), \\
     L(R) &= \mathsf{ess\ inf}_{\vecx \in \cX} \mu_{\vecx}(R),    
 \end{align}
where the essential supremum and infimum are with respect to $(\cX,\cF_{\cX}, \pi)$.
Fix any $\delta>0$ and define
\begin{equation}
    S = \{\vecx \in \cX: \mu_{\vecx}(R) > U(R) - \delta \}, \     S' = \{\vecx \in \cX: \mu_{\vecx}(R) < L(R) + \delta \}. 
\end{equation}
By definition of essential supremum and infimum and using the fact that $\mu$ is a Markov kernel, $S,S' \in \cF_{\cX}^+$.

Using the $\epsilon$-BDP condition, we can obtain
\begin{align}
    U(R) - \delta &\leq \frac{1}{\pi(S)} \int_{\vecx \in S} \mu_{\vecx}(R) d\pi(\vecx) \\
    &\leq \frac{e^{\epsilon}}{\pi(S')} \int_{\vecx \in S'} \mu_{\vecx}(R) d\pi(\vecx) \\
    &\leq e^{\epsilon} ( L(R) + \delta ). 
\end{align}
Since $\delta > 0$ is arbitrary, we get $U(R) \leq e^{\epsilon} L(R)$ $\forall R \in \cF_{\cY}$.

Now, for $R$, define
\begin{equation}
    E(R) := \{\vecx \in \cX: \mu_{\vecx}(R) \notin [L(R),U(R)] \},
\end{equation}
and note that $\pi(E(R)) = 0$ by definition of essential supremum and infimum and using the fact that $\mu$ is a Markov kernel.

We shall now modify $M$ on a judiciously chosen $\pi$-null set to obtain an $\epsilon$-LDP mechanism $M'$.
Since $\cY$ is second countable as it is a Polish space, there exists a countable collection of open sets $\cV = \{V_i \}_{i \geq 1} \subset \cF_{\cY}$ such that every open set $R \in \cF_{\cY}$
 can be written as a disjoint union of some countable subset of $\cV$.

 Define $E  = \cup_{i \geq 1} E(V_i)$ and note that $\pi(E) = 0$ by countable sub-additivity.
 Fix any $\bar\vecx \in E^c$ and define the modification $M'$ of $M$ via
 \begin{equation}
     M'(\vecx) = 
     \begin{cases}
     M(\bar\vecx) & \text{if } \vecx \in E \\
     M(\vecx) & \text{otherwise.}
     \end{cases}
 \end{equation}

 Equivalently, $M'$ is represented by the Markov kernel $\mu': \cX \times \cF_{\cY} \to [0,1]$ defined by
 \begin{equation}
     \mu'_{\vecx}(R) = 
     \begin{cases}
         \mu_{\bar\vecx}(R) & \text{if } \vecx \in E \\
         \mu_{\vecx}(R) & \text{otherwise.}
     \end{cases}
 \end{equation}

Since $\pi(E) = 0$, it is easy to see that $\Pr\{M(\vecx) = M'(\vecx) \} = 1$.
By construction, since $E = \cup_{i \geq 1} E(V_i)$, 
\begin{equation}
    \mu'_{\vecx}(V_i) \leq U(V_i) \leq e^{\epsilon} L(V_i) \leq e^{\epsilon} \mu'_{\vecx'}(V_i) \, \, \forall \vecx,\vecx' \in \cX, V_i \in \cV.
\end{equation}
By second countability, and $\sigma$-additivity, we conclude,
\begin{equation}
    \mu'_{\vecx}(R) \leq e^{\epsilon} \mu'_{\vecx'}(R) \, \, \forall \vecx,\vecx' \in \cX, \forall \text{open }R \in \cF_{\cY}.
\end{equation}
Every second countable space is Radon, so by outer regularity of both measures $\mu_{\vecx}$ and $\mu'_{\vecx}$, we finally obtain
\begin{equation}
     \mu'_{\vecx}(R) \leq e^{\epsilon} \mu'_{\vecx'}(R) \, \, \forall \vecx,\vecx' \in \cX, \forall R \in \cF_{\cY}.
\end{equation}
Hence, $M'$ is $\epsilon$-LDP as desired.
 \end{proof}

\subsection{Proof of \cref{theorem:HT}}\label{sec:HT-proof}
\thmHT*
\begin{proof}
By definition $\alpha(R_i,M,S_i,S_i') = \Pr\{M(\vecx) \in R_i | \vecx \in S_i\}$ and  
\begin{align}
    \Pr\{M(\vecx) \in T_i|x_i \in S_i\} &\leq e^{\delta_i} \Pr\{M(\vecx) \in T_i|x_i \in S_i'\} \ \forall  S_i,S_i' \in \cF_{\cX_i}^+, T_i \in \cF_{\cY} 
\end{align}
Setting $T_i = R_i ^C$, get 
$\alpha(R_i,M,S_i,S_i') + e^{\delta_i}\beta(R_i,M,S_i,S_i') \geq 1$.
One can also switch $S_i,S_i'$ above and set $T_i = R_i$ to get $e^{\delta_i}\alpha(R_i,M,S_i,S_i') + \beta(R_i,M,S_i,S_i') \geq 1$.

The converse is straightforward as well.
\begin{align}
    e^{\delta_i}\alpha(R_i,M,S_i,S_i') + \beta(R_i,M,S_i,S_i') \geq 1 \ &\forall S_i,S_i' \in \cF_{\cX_i}^+, R_i \in \cF_{\cY} \\
    \Leftrightarrow \Pr\{M(\vecx) \in R_i | x_i \in S_i'\} \leq e^{\delta_i} \Pr\{M(\vecx) \in R_i | x_i \in S_i\} \ &\forall S_i,S_i' \in \cF_{\cX_i}^+, R_i \in \cF_{\cY}.
\end{align}
 
\end{proof}

\subsection{Proof of \cref{proposition:LDP_to_BCDP}}\label{sec:ldp-to-bcdp-proof}
\propLdpToBcdp*
\begin{proof}
By \cref{lemma:LDP_Bayesian}, an $\epsilon$-DP algorithm is $\epsilon$-BDP as well.
Recall $\Pr\{M(\vecx) \in R| \vecx \in S\} = \frac{P(S \times R)}{\pi(S)}.$
Thus, $\epsilon$-BDP guarntee ensures that $\forall S,S' \in \cF_{\cX}^+$,
$$ \frac{P(S \times R)}{\pi(S)} \leq e^{\epsilon} \frac{P(S' \times R)}{\pi(S')}.$$
Restricting to the sets of form $S = S_i \times \cX_{-i}$ and $S' = S_i' \times \cX_{-i}$ for $S_i, S_i' \in \cF_{\cX_i}^+$ above yields the desired BCDP result.

\end{proof}

\subsection{Proof of \cref{prop:cdp-to-bdp}}\label{sec:prop:cdp-to-bdp}
\propCdpToBdp*
\begin{proof}
Consider the sets $R \in \cF_{\cY}$ such that $\Pr\{M(\vecx) \in R\} = 0$, then for $S_i, S_i' \in \cF_{\cX_i}^+$, we have $\Pr\{M(\vecx) \in R|x_i \in S_i\} = \Pr\{M(\vecx) \in R|x_i \in S_i'\} = 0$.
Thus, to bound the $\bm\delta$ in $\bm\delta$-BCDP guarantee, we can restrict to $R$ in the set $\cF_{\cY}^+ := \{ R \in \cF_{\cY}| \Pr\{M(\vecx) \in R\} > 0 \}$.

For $S_i,S_i' \in \cF_{\cX_i}^+$, $R \in \cF_{\cY}^+$, 
    \begin{align}
        \frac{\Pr\{M(\vecx) \in R|x_i \in S_i \} }{ \Pr\{M(\vecx) \in R|x_i \in S_i' \} } &= \frac{\Pr\{M(\vecx) \in R,x_i \in S_i \} }{ \pi\{x_i \in S_i \}} \frac{\pi\{x_i \in S_i' \} }{\Pr\{M(\vecx) \in R,x_i \in S_i' \} }.  \label{eq:ratio-Ds}
    \end{align}

For the set $R$ under consideration, define 
\begin{align}
    l(\vecx_{-i}) &= \inf_{x_i \in \cX_i} \mu_{x_i,\vecx_{-i}}(R), \\
    u(\vecx_{-i}) &= \sup_{x_i \in \cX_i} \mu_{x_i,\vecx_{-i}}(R). 
\end{align}
By CDP property, we have $u(\vecx_{-i}) \leq e^{c_i} l(\vecx_{-i})$.
Now note that 
\begin{align}
    \frac{\Pr\{M(\vecx) \in R,x_i \in S_i \} }{ \pi\{x_i \in S_i \}} &= \frac{1}{\pi\{x_i \in S_i \}} \int_{\vecx_{-i} \in \cX_{-i},x_i \in S_i } \mu_{x_i,\vecx_{-i}}(R) d\pi(x_i,\vecx_{-i}) \\
    &\leq \frac{1}{\pi\{x_i \in S_i \}} \int_{\vecx_{-i} \in \cX_{-i},x_i \in S_i } u(\vecx_{-i}) d\pi(x_i,\vecx_{-i}) \\
    &= \int_{\vecx_{-i} \in \cX_{-i}} u(\vecx_{-i}) d\pi_{\vecx_{-i}|x_i \in S_i}(\vecx_{-i}). \label{eq:conc1}
\end{align}
Similarly, we have 
\begin{equation}
    \frac{\Pr\{M(\vecx) \in R,x_i \in S_i' \} }{ \Pr\{x_i \in S_i' \}} \geq 
   \int_{\vecx_{-i} \in \cX_{-i}} l(\vecx_{-i}) d\pi_{\vecx_{-i}|x_i \in S_i'}(\vecx_{-i}).  \label{eq:conc2}
\end{equation}
\end{proof}
Using \eqref{eq:conc1} and \eqref{eq:conc2} in \eqref{eq:ratio-Ds}, we obtain
\begin{align}
    \frac{\Pr\{M(\vecx) \in R|x_i \in S_i \} }{ \Pr\{M(\vecx) \in R|x_i \in S_i' \} } &\leq \frac{\int_{\vecx_{-i} \in \cX_{-i}} u(\vecx_{-i}) d\pi_{\vecx_{-i}|x_i \in S_i}(\vecx_{-i})}{\int_{\vecx_{-i} \in \cX_{-i}} l(\vecx_{-i}) d\pi_{\vecx_{-i}|x_i \in S_i'}(\vecx_{-i})} \\
    &\leq e^{c_i} \frac{\int_{\vecx_{-i} \in \cX_{-i}} l(\vecx_{-i}) d\pi_{\vecx_{-i}|x_i \in S_i}(\vecx_{-i})}{\int_{\vecx_{-i} \in \cX_{-i}} l(\vecx_{-i}) d\pi_{\vecx_{-i}|x_i \in S_i'}(\vecx_{-i})} \\
    &\leq e^{c_i}\left(1 + \mathsf{TV}(\pi_{\vecx_{-i}|x_i \in S_i},\pi_{\vecx_{-i}|X_i\in S_i'}) \frac{\max_{\vecx_{-i}}l(\vecx_{-i})- \min_{\vecx_{-i}}l(\vecx_{-i})}{\min_{\vecx_{-i}}l(\vecx_{-i})} \right) \\
        &\leq e^{c_i}\left(1 + \mathsf{TV}(\pi_{\vecx_{-i}|x_i \in S_i},\pi_{\vecx_{-i}|x_i \in S_i'}) (e^{\sum_{j \neq i} c_j}-1) \right),
\end{align}
where we used $\frac{\max_{\vecx_{-i}}l(\vecx_{-i})}{\min_{\vecx_{-i}}l(\vecx_{-i})} \leq e^{\sum_{j\neq i} c_j}$ by CDP property.

The $\sum_i c_i$-LDP property follows by straightforward composition. 

Finally, notice that, in case that we know the mechanism is $\varepsilon$-LDP, we can bound $\frac{\max_{\vecx_{-i}}l(\vecx_{-i})}{\min_{\vecx_{-i}}l(\vecx_{-i})}$ by $e^\varepsilon$.
In addition, by \cref{proposition:LDP_to_BCDP}, we know $\delta_i' \leq \epsilon$.
Thus, we get $\delta_i' \leq \min\{c_i + \log (1 + q_i e^{\epsilon} - q_i), \epsilon \}$

\subsection{Additional Results for Section \ref{sec:CDPTOBCDP}}
\label{sec:proof:lemma:CDP_to_BCP_approximation}
\begin{restatable}{proposition}{lemCDPToBCPapproximation} \label{lemma:CDP_to_BCP_approximation}
Let 
\begin{equation}
    A_{\bm\delta}[i,j] = \begin{cases}
        \frac{1}{\delta_i} & i=j,\\
        \frac{1}{\log\left( 1 + \frac{e^{\delta_i}-1}{q_i}\right)} & \text{else}.
    \end{cases}
\end{equation}
Then, the condition $A_{\bm\delta}\vecc \leq \bm1$ implies $\delta_i \geq c_i + \log (1 + q_i e^{\sum_{j\neq i} c_j} - q_i)$ for all $i \in [d]$.
\end{restatable}
\begin{proof}
Without loss of generality, we establish the proof for $i=1$. For a given value of $c_1$, we have the constraint that
    \begin{equation} \label{eq:c-constraint}
        \sum_{j=2}^d c_j \leq \log\left( 1 + \frac{e^{\delta_1 - c_1}-1}{q_1}\right).
    \end{equation}
Note that $g(x) = \log\left( 1 + \frac{e^{\delta_1 - x}-1}{q_1}\right)$ is concave and $g(0) = \log\left( 1 + \frac{e^{\delta_1}-1}{q_1}\right)$ and $g(\delta_1) = 0$.
    Thus, $g(x) \geq \left(1-\frac{x}{\delta_1}\right) \log\left( 1 + \frac{e^{\delta_1}-1}{q_1}\right)$.
Therefore, we can ensure \eqref{eq:c-constraint} by imposing
\begin{equation}
    \sum_{j=2}^d c_j \leq \left(1-\frac{c_1}{\delta_1}\right) \log\left( 1 + \frac{e^{\delta_1}-1}{q_1}\right).
\end{equation}   

\end{proof}
\subsection{Proof of \cref{lemma:bdp_data_processing}} \label{sec:PostProc}
\lemPostProc*
\begin{proof}
Let the $\sigma$-field on $\cZ$ be denoted by $\cF_{\cZ}$. For $T \in \cF_{\cZ}, Y(T) := \{ y \in \cY| K(y) \in T\} \in \cF_{\cY} $ since $K$ is a measurable function.
Thus, we have the following for any $i \in [d], T \in \cF_{\cZ}$ and $S_i, S_i' \in \cF_{\cX_i}^+$ by BCDP guarantees on $M$
    \begin{align}
         \Pr\{K(M(\vecx)) \in T| x_i \in S_i\} &= \Pr\{M(\vecx) \in Y(T)| x_i \in S_i\} \\
         &\leq e^{\delta_i} \Pr\{M(\vecx) \in Y(T)| x_i \in S_i'\} \\
         &= e^{\delta_i} \Pr\{K(M(\vecx)) \in T| x_i \in S_i'\}.
    \end{align}

\end{proof}

\section{Private Mean Estimation}

\subsection{Numerical Experiment} \label{sec:mean-exp}
\begin{figure}[h]
    \centering
    \includegraphics[width=0.5\textwidth]{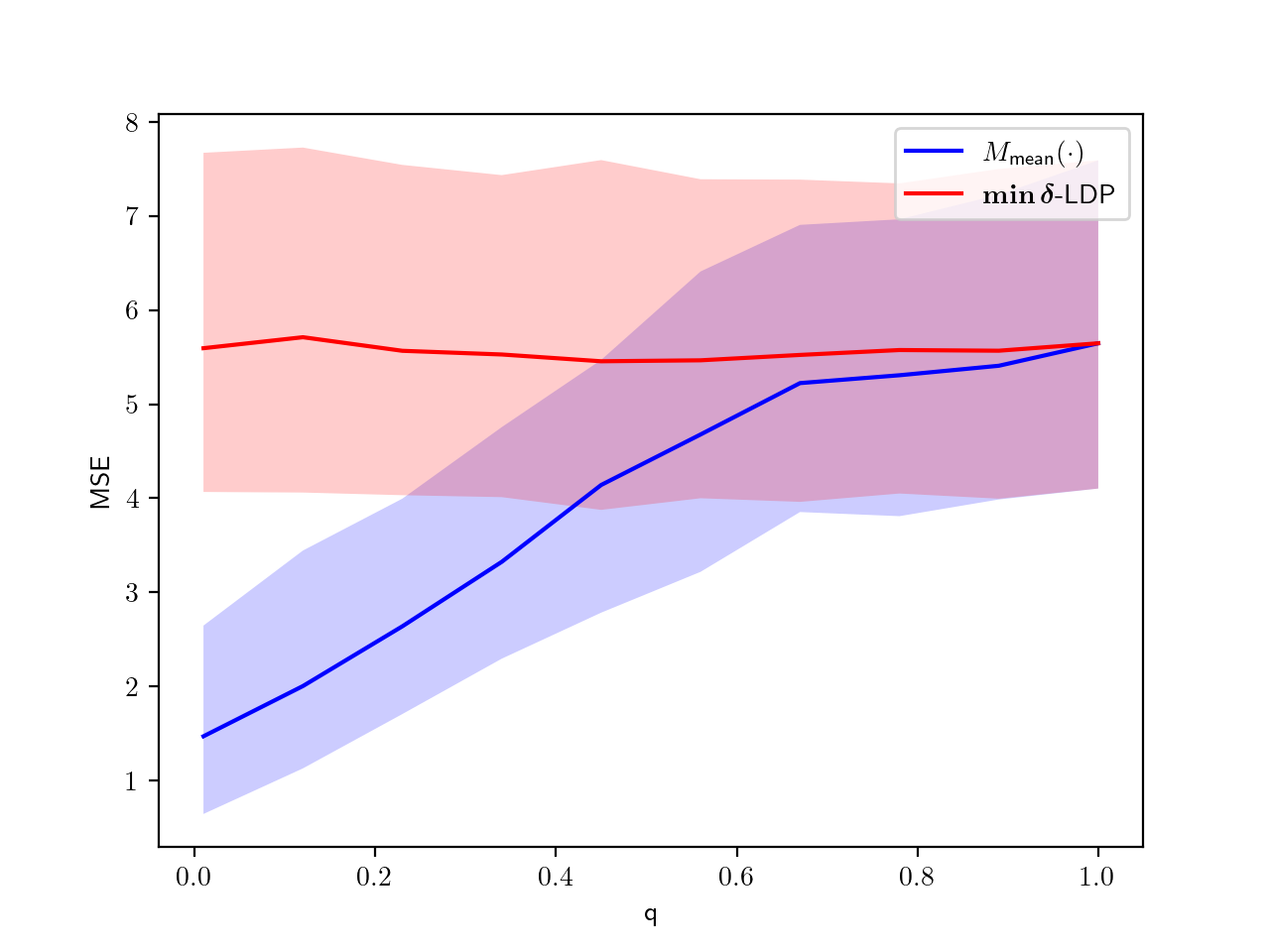}
    \caption{MSE as $q$ is varied keeping $\bm\delta$ and $\varepsilon$ constant for $M_{\mathsf{mean}}$ and $\min\bm\delta$-LDP. }
    \label{fig:exp}
\end{figure}

Here, we provide a simple numerical experiment to validate the performance of our proposed algorithm. For baseline comparison, we consider the $M_{\mathsf{LDP}}$ privacy mechanism from \citet{duchi2013local} with privacy parameter $\min \bm\delta$. For our proposed privacy mechanism $M_{\mathsf{mean}}$, we have a degree of freedom in choosing $\zeta$. We note that, as $q \to 1$, we would want $\zeta \to 1$ since in this case, the best algorithm is to do $\min \bm\delta$-LDP. As $q \to 0$, we would want $c_i = \tilde \delta_i$ as CDP is ideal in this case. Thus, we settle with a heuristic choice $\zeta = (1+q)/2$.

We first construct a distribution for which \cref{assumption:bounded_TV} holds.
We consider the distribution where 
\begin{align}
    Z &\sim \mathsf{Bern}\left(\frac{1}{2}\right), \\
    (x_1,\ldots,x_d) &= \begin{cases}
        (Z,\ldots,Z) & \text{ w.p. } q \\
        \text{i.i.d.}\ \ \mathsf{Bern}\left(\frac{1}{2}\right) & \text{ else.}
    \end{cases}
\end{align}
It is easy to verify that $\mathsf{TV}(\pi_{\vecx_{-i}|x_i = 1},\pi_{\vecx_{-i}|x_i = 0}) = q \ \ \forall i$ in this setting.
We fix $d=10$,  $\bm\delta = (0.2,0.2,\varepsilon,\varepsilon,\ldots,\varepsilon)$, and $\varepsilon=2$. 
This coordinate-wise privacy demand captures the impact of requiring more privacy protection for the first two-coordinates on the algorithm's error. 

We vary the correlation $q$ and plot the median MSE values, along with the $25$-th and $75$-th quantiles, for our proposed algorithm BCDP and the baseline algorithm in \cref{fig:exp}.
We set $n=10000$, and sample user data i.i.d. $2 \times $Bern$(\frac{1}{2})-1$.
The experiment is run over 1000 trials.

When $q \approx 1$, there is a high degree of correlation between the coordinates and our algorithm resembles $\min \bm\delta$-LDP.
For other values of $q$, the proposed BCDP algorithm can leverage the heterogeneous nature of the privacy demand and obtain better performance.

\subsection{Proofs} 

\label{sec:proof:theorem:mean_estimation}
\thmMeanEstimation*

\begin{proof} Our proof has two main parts that are presented below: proof of privacy guarantee, and proof of the error bounds.

\textbf{Proof of privacy:}  Note that we use the channel $M_{\mathsf{LDP}}(.)$ a total of $d$ times with varying size inputs in Mechanism \ref{alg:local-channels}. Focusing on any particular data point $\vecx$, let
\begin{equation}
\tilde{M}(\vecx) := \left (M_{\mathsf{LDP}}(\vecx_{i:d}, c_i - c_{i-1}, \sqrt{d-i+1}) \right )_{i=1}^d.
\end{equation}
First, notice that, by composition, $\tilde{M}$ is $c_d$-LDP. Given that $c_d \leq \tilde{\delta}_d \leq \varepsilon$, this immediately implies that our algorithm is $\varepsilon$-LDP. To show the \bdp bound, given \cref{lemma:bdp_data_processing}, it suffices to show $\tilde{M}$ is $\bm \delta$-\bdp. First, notice that \cref{proposition:LDP_to_BCDP}, along with choice of $c_d \leq \tilde{\delta}_d$, implies that this algorithm is $\delta_d$-\bdp with respect to coordinate $d$. In fact, for any $i$ for which $c_d \leq \tilde{\delta}_i$, a similar argument implies that $\tilde{M}$ is $\tilde{\delta}_i$-BCDP with respect to coordinate $i$. Hence, we could only focus on other coordinates $i < d$ for which $c_d > \tilde{\delta}_i$. 

Next, we claim that $\tilde{M}(\vecx)$ is $\vecc$-CDP.
To see this, notice that, by perturbing the $k$-th coordinate of $\vecx$, the terms affected in $\tilde{M}(\vecx)$ are $M_{\mathsf{LDP}}(\vecx_{i:d}, c_i - c_{i-1}, \sqrt{d-i+1})$ for $i \leq k$; in \cref{fig:algorithm-scheme}, this corresponds to the bottom $i$ rows of the $Y$(s), i.e., $\{Y^k\}_{k \in [i]}$.
Thus, it suffices to show
\begin{equation}
\tilde{M}_k(\vecx) := \left (M_{\mathsf{LDP}}(\vecx_{i:d}, c_i - c_{i-1}, \sqrt{d-i+1}) \right )_{i=1}^k
\end{equation}
is $c_k$-CDP with respect to coordinate $k$. However, this immediately follows  as $\tilde{M}_k$ is $c_k$-LDP, implying it is also $c_k$-CDP with respect to the $k$-th coordinate.

As a result, by the second part of \cref{prop:cdp-to-bdp}, and given that $\tilde{M}$ is $c_d$-LDP and \cref{assumption:bounded_TV} holds, we obtain that $\tilde{M}$ is $c_i + \log(1+qe^{c_d}-q)$-\bdp with respect to coordinate $i$. It remains to show that $c_i + \log(1+qe^{c_d}-q) \leq \delta_i$ for any $i<d$. This also follows from the choice of $c_i$ for $i<d$ and $c_d > \tilde{\delta}_i$, which are the conditions we are operating under. It is worth noting that the condition 
\begin{equation}
c_d \leq \log(\frac{e^{\zeta \tilde{\delta}_1}+q-1}{q})     
\end{equation}
guarantees that the chosen $c_i$'s are nonnegative, given $\zeta \in (0,1)$.

\textbf{Proof of the mean-squared error bound:}

\citet{Duchi13} show that the channel $M_{\mathsf{LDP}}(\cdot)$'s output is unbiased and it holds that for any $\vecx \in [-1,1]^d$, we have
\begin{equation}
    \bbE [ \|M_{\mathsf{LDP}}(\vecx,\varepsilon, d) - \vecx\|^2_2 |\vecx] \leq \cO(d) \left ( \frac{d}{\varepsilon^2} \wedge 1 \right ).
\end{equation}
Note that, due to the symmetry of \cref{alg:duchi-mldp}, the error is identically distributed across each coordinate in expectation. 
Hence, for any vector $\vecx$ and any $k \leq i$, the random variable $Y^{k}_{i-k+1}$ is an unbiased estimators of $\vecx_{i}$ with its variance bounded by $\mathcal{O}(1)  \left ( \frac{d-i+1}{(c_i-c_{i-1})^2} \wedge 1 \right )$.

Note that, the errors from different runs of the algorithm are independent. Therefore, using the Cauchy–Schwarz inequality, we can see that the minimum variance in estimating $\vecx_{i}$ through a linear combination of $Y^{(i,j)}$ is achieved by choosing weights proportional to the inverse of the variances which would result in the following estimator
\begin{equation} \label{eqn:estimator:fixj_fixi}
\frac{\sum_{k=1}^i  Y_{i+1-k}^{k} w_k}{\sum_{k=1}^i w_k} 
\end{equation}
with the error rate given by
\begin{equation}
\mathcal{O}(1) \frac{1}{\sum_{k=1}^i \frac{(c_k - c_{k-1})^2}{d-k+1} } \wedge 1.
\end{equation}
Also, note that, the algorithm's runs over different users are independent. Therefore, taking average of $M_{\mathsf{mean}}(\vecx^{(j)},\vecc)$ over $j$, i.e., different users' data, gives us an estimator of 
\begin{equation}
\frac{1}{n}\sum_{j=1}^n \vecx^{(j)}_{i}    
\end{equation}
with the error rate 
\begin{equation}
\mathcal{O}(1) \frac{1}{n} \cdot \frac{1}{\sum_{k=1}^i \frac{(c_k - c_{k-1})^2}{d-k+1} } \wedge 1.    
\end{equation}
Summing this over the coordinates gives us the following upper bound:
\begin{equation}
\mathcal{O}(1) \frac{1}{n} \sum_{i=1}^d \frac{1}{\sum_{k=1}^i \frac{(c_k - c_{k-1})^2}{d-k+1} } \wedge d.    
\end{equation}
\end{proof}

\CorMse*
\begin{proof}
    For convenience, we set $\zeta = 0.5$. We do not focus on refining the constants in the Mean-Squared Error (MSE) bound with a potentially better choice of $\zeta$.
    Now consider two cases for the correlation bound $q$
    \begin{enumerate}
        \item $q \in [0, \frac{e^{\delta/2}-1}{e^{\epsilon}-1}]$: in this case, we have $c_{k+1} = c_{k+2} = \ldots = c_d = \epsilon$ and $c_1 = c_2 = \ldots = c_k = \delta - \log(1+qe^{\epsilon}-q)$.
        Using the range of $q$, we get $\delta - \log(1+qe^{\epsilon}-q) \in [\delta/2, \delta]$.
        Thus, replacing the $c_1$ to $c_k$ values by the worst-case of $\delta/2$ we get an error of 
        \begin{equation}
            \mathsf{MSE} \lesssim \frac{d}{n} \left( \frac{k}{\delta^2} + \frac{d-k}{\delta^2 + \frac{d}{d-k} (\epsilon - \delta/2)^2 } \right) \wedge d
        \end{equation}
        Using $\epsilon \geq 2 \delta$, we get 
        \begin{equation}
            \mathsf{MSE} \lesssim \left( \frac{dk}{n\delta^2} + \frac{(d-k)^2}{n\epsilon^2} \right) \wedge d
        \end{equation}
        \item $q \in (\frac{e^{\delta/2}-1}{e^{\epsilon}-1},1]$: in this case we have $c_{k+1} = \ldots = c_d  = \log\frac{e^{\delta/2} + q - 1}{q}$ and $c_1 = \ldots = c_k = \delta -  \log(1+qe^{c_d}-q) = \delta/2$. 
         Thus, we get an error of 
        \begin{equation}
            \mathsf{MSE} \lesssim \left( \frac{dk}{n\delta^2} + \frac{(d-k)d}{n \left( \delta^2 + \frac{d}{d-k} (c_d - \delta/2)^2 \right)} \right) \wedge d
        \end{equation}
    \end{enumerate}
    
\end{proof}

\section{Least-Squares Regression}
\label{sec:appendix-LSE}

\subsection{Proof of \cref{theorem:LSE}} \label{proof:TheoremLSE}
\thmLSE*
\begin{proof}
\textbf{Proof of privacy:} The proof of the privacy guarantee follows a similar structure to the proof of \cref{theorem:mean_estimation}. Note that while composition does not hold for BCDP, it does hold for LDP and CDP. This is why we divide $\vecc$ by two, rather than dividing $\bm{\delta}$ by two. Additionally, since the OPT algorithm only observes the private copies of the data, by the post processing inequality, it does not matter how the algorithm finds the minimizer (e.g., how many passes it makes over each private data point), as the privacy guarantee remains unchanged.

\textbf{Proof of error rate:} Without loss of generality, we can ignore the constant term in $f(\cdot)$ that does not depend on $\btheta$, and hence, with a slight abuse of notation, we can rewrite $f(\cdot)$ as
\begin{equation}
{f}(\btheta) := \frac{1}{2n} \sum_{i=1}^n \left (\btheta^\top \bm{{z}}^{(i)}\bm{{z}}^{{(i)}^\top} \btheta - 2\btheta^\top  {l}^{(i)} \bm{{z}}^{(i)} \right ).    
\end{equation}
Also, recall that $\hat{f}(\cdot)$ is given by
\begin{equation}
\hat{f}(\btheta) := \frac{1}{2n} \sum_{i=1}^n \left (\btheta^\top \bm{\hat{z}}^{(i,1)}\bm{\hat{z}}^{{(i,2)}^\top} \btheta - 2\btheta^\top  \hat{l}^{(i,1)} \bm{\hat{z}}^{(i,2)} \right ).    
\end{equation}
Denote the minimizer of $\hat{f}(\btheta)$ over $\Theta$ by $\btheta^*(\{\bm{\hat{x}}^{(i,1)}, \bm{\hat{x}}^{(i,2)}\}_{i=1}^n)$. Also, we denote the output of \cref{alg:LSE_BCDP} by $\bm{\hat{\theta}}(\{\bm{\hat{x}}^{(i,1)}, \bm{\hat{x}}^{(i,2)}\}_{i=1}^n)$. Finally, recall that $\btheta^*$ denotes the minimizer of function $f$.

\textit{For convenience and brevity, we drop the conditioning on $\vecx^{(1)},\ldots,\vecx^{(n)}$ from the notation, but all expectations henceforth are conditioned on $\vecx^{(1)},\ldots,\vecx^{(n)}$.}

Next, note that
\begin{enumerate}
\item  By the definition of OPT, we have
\begin{equation} \label{eqn:LSEproof_eq1}
\hat{f} \left (\bm{\hat{\theta}}(\{\bm{\hat{x}}^{(i,1)}, \bm{\hat{x}}^{(i,2)}\}_{i=1}^n)  \right) - \hat{f} \left( \btheta^*(\{\bm{\hat{x}}^{(i,1)}, \bm{\hat{x}}^{(i,2)}\}_{i=1}^n)\right)  \leq \frac{1}{n}. 
\end{equation}
\item By the definition of $\btheta^*(\{\bm{\hat{x}}^{(i,1)}, \bm{\hat{x}}^{(i,2)}\}_{i=1}^n)$, we have
\begin{equation}\label{eqn:LSEproof_eq2}
\hat{f} \left( \btheta^*(\{\bm{\hat{x}}^{(i,1)}, \bm{\hat{x}}^{(i,2)}\}_{i=1}^n)\right) - \hat{f}(\btheta^*) \leq 0.    
\end{equation}
\item Given the construction of $\hat{f}$, it is an unbiased estimator of $f$ for any fixed $\btheta$. Hence, we have
\begin{equation}\label{eqn:LSEproof_eq3}
\mathbb{E}[\hat{f}(\btheta^*) - f(\btheta^*)] = 0.    
\end{equation}
\end{enumerate}
Summing all the three, given us
\begin{equation}
\mathbb{E} \left [ \hat{f} \left (\bm{\hat{\theta}}(\{\bm{\hat{x}}^{(i,1)}, \bm{\hat{x}}^{(i,2)}\}_{i=1}^n)  \right) - f(\btheta^*) \right ]    \leq \frac{1}{n}.
\end{equation}
However, we need to bound 
\begin{equation}
\mathbb{E} \left [ {f} \left (\bm{\hat{\theta}}(\{\bm{\hat{x}}^{(i,1)}, \bm{\hat{x}}^{(i,2)}\}_{i=1}^n)  \right) - f(\btheta^*) \right ].
\end{equation}
Therefore, it suffices to bound
\begin{equation} \label{eqn:LSEproof_eq4}
\mathbb{E} \left [ {f} \left (\bm{\hat{\theta}}(\{\bm{\hat{x}}^{(i,1)}, \bm{\hat{x}}^{(i,2)}\}_{i=1}^n)  \right) - \hat{f} \left (\bm{\hat{\theta}}(\{\bm{\hat{x}}^{(i,1)}, \bm{\hat{x}}^{(i,2)}\}_{i=1}^n)  \right) \right ].    
\end{equation}
Observe that this term is upper bounded by 
\begin{equation} \label{eqn:LSEproof_eq5}
\mathbb{E} \left [ \sup_{\btheta} (f(\btheta) - \hat{f}(\btheta))  \right ].    
\end{equation}
Note that \eqref{eqn:LSEproof_eq5} is upper bounded by 
\begin{equation} \label{eqn:LSEproof_eq6}
\sup_{\btheta \in \Theta} \|\theta\|^2 
\mathbb{E} \left [ \left \| \frac{1}{n} \sum_{i=1}^n \bm{{z}}^{(i)}\bm{{z}}^{{(i)}^\top} - \bm{\hat{z}}^{(i,1)}\bm{\hat{z}}^{{(i,2)}^\top} \right \|\right ]
+
\sup_{\btheta \in \Theta} \|\theta\| 
\mathbb{E} \left [ \left \| \frac{1}{n} \sum_{i=1}^n  {l}^{(i)} \bm{{z}}^{(i)} - \hat{l}^{(i,1)} \bm{\hat{z}}^{(i,2)} \right \| \right ].
\end{equation}
We bound the first term above. The second term can be bounded similarly. To do so, 
first note that we can decompose $ \bm{\hat{z}}^{(i,1)}\bm{\hat{z}}^{{(i,2)}^\top} - \bm{{z}}^{(i)}\bm{{z}}^{{(i)}^\top}$ as
\begin{equation}
(\bm{\hat{z}}^{(i,1)} - \bm{{z}}^{(i)}) \bm{{z}}^{{(i)}^\top} 
+ \bm{{z}}^{(i)} (\bm{\hat{z}}^{(i,2)} - \bm{{z}}^{(i)})^\top 
+ (\bm{\hat{z}}^{(i,1)} - \bm{{z}}^{(i)})(\bm{\hat{z}}^{(i,2)} - \bm{{z}}^{(i)})^\top.  
\end{equation}
Hence, it suffices to bound the expected norm of the average of each term separately, and all of them can be handled similarly. We will focus on the last term in particular, i.e., we will bound
\begin{equation}\label{eqn:LSEproof_eq7}
\mathbb{E} \left [ \left \| \frac{1}{n} \sum_{i=1}^n A^i \right \|\right ] 
\end{equation}
with $A^i:= (\bm{\hat{z}}^{(i,1)} - \bm{{z}}^{(i)})(\bm{\hat{z}}^{(i,2)} - \bm{{z}}^{(i)})^\top$.
Next, we make the following claim regarding matrices $\{A^i\}_{i=1}^n$, proved in \cref{sec:pf-matrix}.
\begin{lemma}\label{lemma:matrix}
Let 
\begin{equation}
r^2 := \sup_{\vecx\in[-1,1]^d} \mathbb{E}[\|M_{\mathsf{mean}}(\vecx,\vecc/2)-\vecx\|^2|\vecx].    
\end{equation}
Then, the following two results hold:
\begin{align}
\| \sum_i \mathbb{E}[A^i A^{i^\top}] \| &\leq  n r^4,  \\
\mathbb{E}[\max_i \|A^i\|^2 ] & \leq n r^4.
\end{align}
\end{lemma}
Let us first show how this lemma completes the proof. Using Theorem 1 in \cite{tropp2016expected}, Jensen's inequality, along with this lemma, we can immediately upper bound \eqref{eqn:LSEproof_eq7} by 
\begin{equation}
\mathcal{O}(1) \log(d) \frac{r^2}{\sqrt{n}}.    
\end{equation}
Similarly, the terms $\|(\bm{\hat{z}}^{(i,1)} - \bm{{z}}^{(i)}) \bm{{z}}^{{(i)}^\top}\|$ and $ \|\bm{{z}}^{(i)} (\bm{\hat{z}}^{(i,2)} - \bm{{z}}^{(i)})^\top\|$ can be upper bounded by $\cO(\frac{\log(d)rd}{\sqrt{n}})$ using similar steps as proof of \cref{lemma:matrix}.

Recall  $r^2 = \mathcal{O}(1) \sum_{i=1}^d \frac{1}{\sum_{k=1}^i \frac{(c_k - c_{k-1})^2}{d-k+1} }$.  Note that we use $M_{\mathsf{mean}}$ without the truncation to $[-1,1]^d$ to preserve unbaisedness, and hence, the minimum with $d$ does not appear in $r^2$, unlike \cref{theorem:mean_estimation}.
Thus we obtain the bound
\begin{equation} 
\mathbb{E} \left [ \sup_{\btheta} (f(\btheta) - \hat{f}(\btheta))  \right ] \lesssim \max_{\btheta \in \Theta} (\|\btheta\|^2+ \|\btheta\|) \left( \frac{r^2 \log d }{\sqrt{n}} + \frac{r d\log d }{\sqrt{n}} \right).
\end{equation}
We can also project the estimate $\hat{\theta}$ to the set $\Theta$ to obtain the bound $\max_{\btheta \in \Theta} (\|\btheta\|^2+ \|\btheta\|)d$.
Taking minimum of the terms, we get \cref{theorem:LSE}.

\subsection{Proof of \cref{lemma:matrix}}
\label{sec:pf-matrix}
First, note that
\begin{equation}
A^i A^{i^\top} = \|\bm{\hat{z}}^{(i,2)} - \bm{{z}}^{(i)}\|^2 (\bm{\hat{z}}^{(i,1)} - \bm{{z}}^{(i)})(\bm{\hat{z}}^{(i,1)} - \bm{{z}}^{(i)})^\top,      
\end{equation}
which yields
\begin{equation}
\|\mathbb{E}[A^i A^{i^\top}]\| \leq r^2 \mathbb{E}[\|(\bm{\hat{z}}^{(i,1)} - \bm{{z}}^{(i)})(\bm{\hat{z}}^{(i,1)} - \bm{{z}}^{(i)})^\top\|].    
\end{equation}
Using the result that $\|uu^\top\| = \|u\|^2$ for any vector $u$ completes the proof of the first part of the lemma. To prove the second part, note that
\begin{equation}
\mathbb{E}[\max_i \|A^i\|^2 ] \leq
\mathbb{E}[\sum_{i=1}^n \|A^i\|^2] = \sum_{i=1}^n \mathbb{E}[\|A^i\|^2]. 
\end{equation}
Using the result that $\|uv^\top\| \leq \|u\| \|v\|$ for any two vectors $u$ and $v$ implies
\begin{equation}
\mathbb{E}[\|A^i\|^2] \leq 
\mathbb{E}\left[ \|\bm{\hat{z}}^{(i,1)} - \bm{{z}}^{(i)}\|^2 \|\bm{\hat{z}}^{(i,2)} - \bm{{z}}^{(i)}\|^2\right ].
\end{equation}
Finally, the conditional independence of $\bm{\hat{z}}^{(i,1)} $ and $\bm{\hat{z}}^{(i,2)} $ given $z^{(i)}$ completes the proof of lemma. 
\end{proof}

\end{document}